\documentclass[a4paper,preprint,9pt]{imsart}
\usepackage{latexsym}
\usepackage{graphicx,grffile} 
\graphicspath{{./figures/}}
\usepackage{mathrsfs} 
\usepackage{amsmath,amssymb,amsthm}  
\usepackage[american]{babel}
\usepackage[colorlinks,citecolor=blue,urlcolor=blue]{hyperref}
\usepackage{bm}
\usepackage{dsfont}
\usepackage{subcaption}
\usepackage[T1]{fontenc}
\usepackage[utf8]{inputenc}
\usepackage{lmodern}
\usepackage[ruled,linesnumbered]{algorithm2e}
\usepackage{float}
\usepackage{natbib} 
\usepackage{mathtools}
\usepackage{color}
\usepackage[table]{xcolor}
\usepackage{booktabs} 
\usepackage{enumitem}  

\usepackage{mdframed}

\usepackage{comment}

\usepackage[top=1.0in, bottom=0.9in, left=1.4in, right=1.6in]{geometry}

\usepackage{custom_style}

\theoremstyle{definition} 
\DeclareMathOperator{\card}{card}

\def\*#1{\mathbf{#1}}

\newcommand{\indep}{\perp \!}
\newcommand{\betalasso}{\hat{\boldsymbol{\beta}}^{\text{LASSO}}}
\newcommand{\betamle}{\hat{\boldsymbol{\beta}}_{\lambda}}
\newcommand{\betadx}{\hat{\boldsymbol{\beta}}^{d_{\bX_{*,j}}}}

\newcommand{\betady}{\hat{\boldsymbol{\beta}}^{d_{y},j}}

\newcommand{\betatrue}{\bm\beta^0}
\newcommand{\betanull}{\hat{\bm\beta}^{\text{null}}}
\newcommand{\wjtrue}{\bw^{0,j}}

\newcommand{\Sbycrt}{\widehat{\cS}_{\text{BY-CRT}}}

\newcommand{\hatS}{\widehat{\cS}}
\newcommand{\hatP}{\widehat{p}}
\newcommand{\hatK}{\widehat{k}}
\newcommand{\betalogreg}{\hat{\bm\beta}^{\texttt{logreg}}}

\begin{document}
\begin{frontmatter}
  \title{A Conditional Randomization Test for Sparse Logistic Regression in
    High-Dimension}
  \runtitle{\texttt{CRT-logit}}
  \date{}
  
  \begin{aug}
    \author[A]{Binh T. Nguyen},
    \author[B]{Bertrand Thirion},
    \author[C]{Sylvain Arlot}
      \runauthor{Nguyen, Thirion, Arlot}
      \address[A]{T\'elecom Paris, Palaiseau 91120, France\\
        \href{mailto:tuanbinhs@gmail.com}{tuanbinhs@gmail.com}}
      \address[B]{Universit\'e Paris-Saclay, Inria, CEA, Palaiseau 91120,
        France\\ \href{mailto:bertrand.thirion@inria.fr}{bertrand.thirion@inria.fr}}
      \address[C]{Universit\'e Paris-Saclay, CNRS, Inria, Laboratoire de
        math\'ematiques d’Orsay, 91405, Orsay,
        France\\ \href{mailto:sylvain.arlot@universite-paris-saclay.fr}{sylvain.arlot@universite-paris-saclay.fr}}
    \end{aug}
\maketitle
\begin{abstract}
  Identifying the relevant variables for a classification model with correct
  confidence levels is a central but difficult task in high-dimension.
  Despite the core role of sparse logistic regression in statistics and machine
  learning, it still lacks a good solution for accurate inference in the
  regime where the number of features $p$ is as large as or larger than the
  number of samples $n$.
  Here we tackle this problem by improving the Conditional Randomization Test
  (CRT).
  The original CRT algorithm shows promise as a way to output p-values while
  making few assumptions on the distribution of the test statistics.
  As it comes with a prohibitive computational cost even in mildly
  high-dimensional problems, faster solutions based on distillation have been
  proposed.
  Yet, they rely on unrealistic hypotheses and result in low-power solutions.
  To improve this, we propose \emph{CRT-logit}, an algorithm that combines a
  variable-distillation step and a decorrelation step that takes into account
  the geometry of $\ell_1$-penalized logistic regression problem.
  We provide a theoretical analysis of this procedure, and demonstrate its
  effectiveness on simulations, along with experiments on large-scale
  brain-imaging and genomics datasets.
\end{abstract}
\end{frontmatter}
\section{Introduction}
\label{sec:intro-dcrt}
%
Logistic regression is one of the most popular tools in modern applications of
statistics and machine learning, partly due to its relative algorithmic
simplicity.
The method belongs to the class of \emph{generalized linear models} that handle
discrete outcomes, \ie classification problems.
Here, we focus on the binary classification problem, where one observation of
the responses $y \in \{0, 1 \}$ and the data vectors $\bx \in \bbR^p$ follows
the relationship:
\begin{equation}
  \label{eq:logit}
  \bbP(y = 1 \mid \bx) = g(\bx^T\bm\beta^0) = \dfrac{1}{1 + \exp(-\bx^T\bm\beta^0)},
\end{equation}
where $g(x) = 1 / (1 + \exp(-x))$ is the sigmoid function, and $\bm\beta^0$ the
vector of true regression coefficients.
In the classical setting, in which the number of samples $n$ is greater than
the number of features $p$, an estimate $\hat{\bm\beta}$ of the true signals
$\bm\beta^0$ can be obtained using \emph{maximum likelihood estimation} (MLE).
The asymptotic behaviour and derivation of the test statistic, confidence
intervals and p-values of the MLE have been well studied, \eg in
\citet{cox_theoretical_1979}.
The availability of p-values for the test statistics makes it possible to rely
on multiple hypothesis testing, where one wants to test which variables have a
non-zero effect on the outcome, \emph{conditionally} to the remaining
variables, \ie
\begin{equation*}
\text{(null) }\cH_0^j: x_j \indep y \mid \bx_{-j} \quad \text{vs.}  \quad
\text{(alternative) } \cH_{\alpha}^j: x_j \not\indep y \mid \bx_{-j},
\end{equation*}
for each feature $j \in [p] \egaldef \{1, \dots, p\}$ and
$\bx_{-j} \egaldef \{x_1, x_2, \dots, x_{j-1}, x_{j+1}, \dots, x_p\}$.
Equivalently, under the setting in Eq.~\eqref{eq:logit}, we have:
\begin{equation*}
\text{(null) }\cH_0^j: \beta^0_j = 0 \quad \text{vs.}  \quad
\text{(alternative) } \cH_{\alpha}^j: \beta^0_j \neq 0.
\end{equation*}
Unfortunately, this line of analysis cannot be applied to the high-dimensional
regime, where $p$ is larger than $n$, as argued in
\cite{sur_modern_2019,yadlowsky_sloe_2021,zhao_asymptotic_2020}.
These works show that in the regime where
$\lim_{n, p \to \infty} n/p = \kappa$, the MLE estimator exists only when
$\kappa > 2$.
However, we note that this type of analysis is done \emph{without} the addition
of $\ell_1$-regularization to the likelihood function, \ie without using a
\emph{penalized estimator} to enforce sparsity.
\paragraph{Motivation}
Our focus in this paper is to do inference with statistical guarantees on
high-dimensional sparse logistic regression, where $p$ is larger or much larger
than $n$.
This setting is typical in modern applications of pattern recognition, \eg in
brain-imaging or genomics \citep{bach_optimization_2012}, with $p$ as large as
hundreds of thousands --compressible to thousands-- but $n$ stays at most few
thousand.
\begin{figure}[h]
  \centering
  \includegraphics[width=0.5\textwidth]{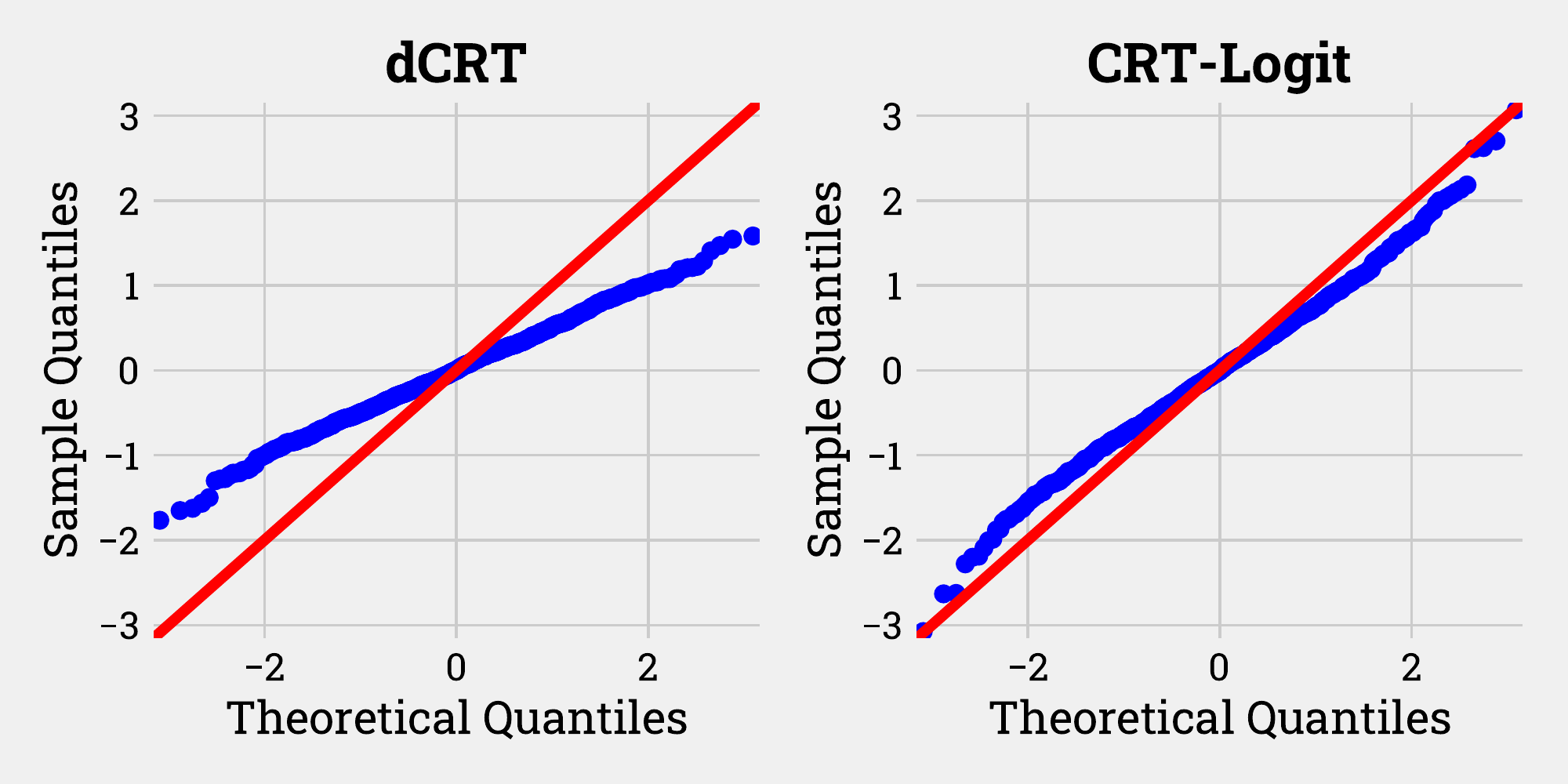}
  \caption{\textbf{QQ-Plot for 1000 samples of test-statistic of a null index
      for logistic regression}, with simulated data, $n=200,p=400$.
  \emph{Left:} Statistics obtained from running Distilled-CRT, and
  \emph{Right:} from our proposed algorithm.
  The empirical distribution of dCRT null-statistic strays far from theoretical
  distribution, which is standard normal, while empirical distribution of
  CRT-logit's null test score is much closer.}
  \label{fig:motivation}
\end{figure}  
The family of methods we consider is the \emph{Conditional Randomization Test
  (CRT)} \citep{candes_panning_2018}.
CRT relies on generating multiple noisy copies of original variables to output
empirical p-values in high-dimensional inference problems.
However, prohibitive computational cost makes CRT impractical, as discussed at
length in
\cite{candes_panning_2018,tansey_holdout_2018,berrett_conditional_2019,liu_fast_2020}.
There have been several lines of research attempting to fix this problem,
most notably the \emph{distilled Conditional Randomization Test (dCRT)}
\citep{liu_fast_2020}.
This work introduced a \emph{distillation step} as a replacement for the
randomized sampling step to compute the importance statistics (see
Section~\ref{sec:background} for more details).
It provides a way to output p-values for multiple types of regression and
classification problems, assuming convergence to Gaussian of the test statistic in
large-sample regime.
Yet, as shown in the left panel of Figure~\ref{fig:motivation}, the originally
proposed dCRT test-statistic for logistic regression does not behave as well as
intended.
\emph{In particular, its null distribution deviates markedly from standard normal in
  high-dimension whenever $n/p  \leq 1$.}
\paragraph{Contribution}
We propose a correction for the dCRT, inspired by the decorrelation
method presented in \cite{ning_general_2017}.
The decorrelation step makes the null-distribution of the test statistics much
closer to standard normal, as shown on the right panel of
Figure~\ref{fig:motivation}, and thus increases the statistical power of the
method.
We provide asymptotic analysis of this method, which shows that CRT-logit
produces standard normal test-statistic in large-sample regime.
In addition, we validate the high performance of CRT-logit on
large-scale brain-imaging and genetics datasets, thus showing its usefulness in
practical applications.
\paragraph{Related works}
The closest cousin of the Conditional Randomization Test is Knockoff Filter
\citep{barber_controlling_2015,candes_panning_2018}, a recent breakthrough in
the False Discovery Rate (FDR) control literature.
It relies on the creation of additional noisy features, called knockoffs, to
calculate variable-importance statistics.
Another extension of vanilla CRT is the Holdout Randomization Test
(HRT) \citep{tansey_holdout_2018}.
While still requiring multiple sampling of noisy variables, HRT solves the
computational issue of original CRT by doing heavy model fitting only once on
one part of the dataset, and test statistics calculation on the other part,
without refitting the model.
However, this method relies on sample-splitting, hence inherently suffers from
a loss of statistical power.
A parallel line of work has introduced the Conditional Permutation Test (CPT)
\citep{berrett_conditional_2019}, a non-parametric alternative to CRT that
relies on a random shuffling mechanism applied to original variables, instead
of multiple sampling of new variables.
This potentially makes CPT more robust to model mis-specification.
\cite{yadlowsky_sloe_2021} recently proposed a method called SLOE, which adapts
the analysis of \cite{zhao_asymptotic_2020}, but in the regime different from
what we are considering, where
$\lim_{n, p \to \infty} n/p \to \kappa \in (1 , 2)$, and more importantly
without sparsity-inducing penalty.
On a separate note, we notice the similarity of dCRT \citep{liu_fast_2020} with
debiased Lasso
\citep{javanmard_confidence_2014,van_de_geer_asymptotically_2014,zhang_confidence_2014}.
This line of work proposed a debiasing formula for the estimator, which makes 
the asymptotic distribution of $(\betalasso - \bm\beta^0)$ standard normal, so that
one can compute the test statistic and p-value associated with each variable.
%
\section{Background}
\label{sec:background}

\paragraph{Notation}
We denote matrices, vectors, scalars and sets by bold uppercase, bold lowercase, script
lowercase , and calligraphic letters, respectively, \eg $\bX$, $\bx$, $x$, and $\cX$.
The $i$-th row of a matrix $\bX$ will be denoted $\bX_{i, *} \ $, the $j$-th
column $\bX_{*, j}$ and the $(i, j)$-th element $\bX_{i, j}$.
For any natural number $p$, we denote the set $[p] \egaldef \{1, \dots, p\}$.
For each $\bx \in \mathbb{R}^p$ and $j \in [p]$, we denote
$\bx_{-j} \egaldef \{x_1, x_2, \dots, x_{j-1}, x_{j+1}, \dots, x_p\}$ a $p-1$
dimensional observation after removing the $j$-th variable.
Correspondingly, $\bX_{-j}$ is the data matrix $\bX \in \bbR^{n\times p}$ with
column $\bX_{*, j}$ removed.
The cumulative distribution function (CDF) of the standard Gaussian
distribution will be denoted $\Phi(\cdot)$.
The indicator function of a random event $\cA$ will be denoted
$\textbf{1}_{\cA}$.
For any two positive sequences $x_n$ and $y_n$, we write $x_n \asymp y_n$ if
$cy_n \leq x \leq Cy_n$ for all $n$, for some positive constants $c$ and
$C$.
For a vector $\bx$, $\norm{\bx}_p$ denotes its $\ell_p$ norm.
For a function $f: \bbR^p \to \bbR$, $\nabla_jf$ denotes its gradient \wrt the
$j$-th variable, for $j \in [p]$.
\paragraph{Problem setting}
We consider exclusively binary classification, where the response vector is
denoted $\by \in \{0, 1\}^n$ and the data matrix $\bX\in \mathbb{R}^{n \times p}$
consists of $n$ $p$-dimensional samples.
Throughout the paper, we assume the data $\{\bX_{i, *}\}_{i=1}^n$ are \iid and
follow a distribution with zero mean and population covariance matrix $\bm\Sigma$.
Moreover, we assume that $\bX_{i, *}$ and $\by_i$ follow the logistic
relationship in Eq.~\eqref{eq:logit}.
We denote the support set $\cS \egaldef \{j \in [p]: {\bm \beta}^0_j \neq 0 \}$ and
assume that it is sparse, \ie $\card(\cS) = s^* \ll p$, where $\card$ denotes the
cardinality of a set.
Furthermore, $\hat{\cS} \egaldef \{j \in [p]: {\bm \hat{\beta}}_j \neq 0 \}$
indicates an estimation of $\cS$, where ${\bm \hat{\beta}}_j$ is an estimate of the
true signal ${\bm \beta}_j^0$.
We try to obtain it through a $\ell_1$-penalized logistic estimator:
\begin{align}
  \label{eq:mle} \betamle
  &= \argmin_{\bm\beta \in \mathbb{R}^{p}} \ \ell(\bm\beta) + \lambda
    \norm{\bm\beta}_1 \ , \ 
  \text{with} \ \ell(\bm\beta) = 
    -\dfrac{1}{n} \sum_{i=1}^n \left\{y_i(\bX_{i, *} \bm\beta) -
    \log\left[1 + \exp(\bX_{i, *}\bm\beta)\right] \right\} \ .
\end{align}
We denote $\bI \egaldef \bbE_{\betatrue} [\nabla^2 \ell(\bm\beta^0)]$ the 
Fisher information matrix, and $\bI_{j \mid -j}$ the 
partial Fisher information, defined by
$\bI_{j \mid -j} \egaldef
\bbE[\nabla^2_{jj} \ell(\bm\beta^0) -
[\nabla^2_{j,-j}\ell(\bm\beta^0)]^{\top}[\nabla^2_{-j,-j}
\ell(\bm\beta^0)]^{-1}\nabla^2_{-j,j}\ell(\bm\beta^0)] = \bI_{jj} -
\bI_{j,-j}\bI_{-j,-j}^{-1}\bI_{-j,j} \ , $
where $\bI_{j, -j}$ is the row-vector made with the $j$th-row and the columns
corresponding to $\bm\beta_{-j}$, $\bI_{-j, -j}$ the sub-matrix of $\bI$ made
with the rows and columns corresponding to $\bm\beta_{-j}$.
This quantity, defined following \citet[pp. 323]{cox_theoretical_1979},
plays an important role in our proposed method, detailed in
Section~\ref{sec:algo}.
\paragraph{Statistical control with False Discovery Rate}
To quantify statistical errors, we consider the \emph{False Discovery
  Rate}, introduced in \cite{benjamini_controlling_1995}.
Given an estimate of the support $\hat{\cS}$, the false discovery proportion
(FDP) is the ratio of the number of selected features that do not belong to the
true support $\cS$, divided by the total number of selected features.
The False Discovery Rate is the expectation of the FDP:
\begin{align*}\label{eq:fdp-fdr}
  \mathrm{FDP}(\hat{\cS})
  = \frac{\card(\{j: j \in \hat{\cS}, j \notin S\})}{\card(\hat{\cS}) \vee 1 } 
  \qquad \text{and} \qquad 
  \mathrm{FDR}(\hat{\cS})
  = \bbE[\mathrm{FDP} (\hat{\cS}) ].
\end{align*}
\paragraph{Conditional Randomization Test (CRT) and Distillation CRT (dCRT)}
The concept of Conditional Randomization Test was originally proposed in the
model-X knockoff paper \citep{candes_panning_2018} as a way to output valid
empirical p-values using knockoff variables.
The principle of the knockoff filter is first to sample noisy copies
$\tilde{\bX}_{*, j}$ of variable $\bX_{*, j}$, given a known sampling mechanism
$P_{j \ \mid -j}$.
One advantage of the knockoff filter is that no specific assumption is placed on
the distribution of the inferred test statistic.
However, this means that, in general, there is no mechanism to derive
p-values from the knockoff statistic.
This motivates the introduction of CRT, which requires running
high-dimensional inference for each variable $j$  $B$ times.
However, the computation cost of CRT is prohibitive when $p$ grows large:
assuming that we use the Lasso program with coordinate descent to compute
$T_j^{\text{CRT}}$, its runtime would be $\cO(Bp^4)$
\cite[pp.~93]{hastie_elements_2009}.
Moreover, CRT requires decently large $B$ to make the empirical distribution of
the p-values smooth enough.
Reducing the computational cost of CRT is the main motivation of several works
\citep{berrett_conditional_2019,liu_fast_2020,tansey_holdout_2018}.
One of them is the introduction of distillation-CRT (dCRT) by 
\citet{liu_fast_2020}.
The main appeal of this method is that it can output p-values analytically,
therefore bypassing the multiple knockoffs sampling steps used to infer on each
variable, and leads to a reasonable reduction of the computation cost.
\paragraph{Distillation operation}
The key addition of dCRT is the distillation operation: for each variable $j$,
we want to distill all the conditional information of the remaining variables
$\bX_{-j}$ to $\bx_j$ and to $\by$ via least-squares minimization with
$\ell_1$-regularization to enforce sparsity.
For each variable $j$, we first solve the lasso problem by regressing
$\bX_{*,j}$ on $\mathbf{X}_{-j}$,
\begin{equation}
  \label{eq:beta-dx}
  \betadx
  = \argmin_{\bm\beta \in \mathbb{R}^{p-1}}
  \dfrac{1}{2} \norm{\bX_{*,j} - \bX_{-j}\bm\beta}_2^2 + \lambda_{dx}\norm{\bm\beta}_1 .
\end{equation}
For distillation of variable $j$ and the binary response $\by$ with logistic
relationship, \cite{liu_fast_2020} briefly suggested to solve a penalized
estimation problem, similar to Eq.~\eqref{eq:mle}:
\begin{equation}
  \label{eq:beta-dy}
  \betady
  = \argmin_{\bm\beta \in \mathbb{R}^{p-1}} -\dfrac{1}{n} \sum_{i=1}^n \left\{y_i(\bX_{i, -j}^{\top}
    \bm\beta) -  \log\left[1 + \exp(\bX_{i, -j}^T\bm\beta)\right]
  \right\} + \lambda \norm{\bm\beta}_1.
\end{equation}
Then, a test statistic is calculated for each $j = 1, \dots, p$:
\begin{equation}
  \label{eq:test-score-dcrt}
  T_j = 
  \sqrt{n} \ \dfrac{\inner{\by - \bX_{-j}\betady, \bX_{*, j} - \bX_{-j}\betadx }}{
    \lVert {\mathbf{y} - \bX_{-j} \betady} \rVert_2
    \lVert {\bx_j - \bX_{-j}\betadx} \rVert_2
  },
\end{equation}
which, under the null hypothesis, and more importantly, assuming linear
relationship between $\bX_{i, *}$ and $\by$, follows standard normal
distribution asymptotically, conditional to $\by$ and $\bX_{-j}$.
It then follows that we can output a p-value for each variable $j$ by taking
$\hatP_j = 2\left[1 - \Phi \left( T_j \right) \right]$.

However, the formulation of test statistics in Eq.~\eqref{eq:test-score-dcrt}
is not truly satisfactory in the setting of sparse logistic regression.
More specifically, both the calculation of regression residuals
$\by - \bX_{-j}\betady$ and test statistics $T_j$ \emph{do not take into
  account the non-linear relationship} between $\bX$ and the binary response
$\by$.
The first row of Figure~\ref{fig:qqplot} plots the qq-plot of the test
statistics $T_j$ for logistic regression, which shows that even in the
classical regime where $n > p$, its distribution is far from standard normal.
\section{Decorrelating Test-Statistics for High-Dimensional Logistic Regression}
\label{sec:algo}
As we have elaborated, the formulation of dCRT is not well-suited for problems
other than penalized least-squares regression.
We therefore propose an adaptation of dCRT in the case of logistic regression,
inspired by the classical work of \cite{cox_theoretical_1979} and by
\cite{ning_general_2017}.
First, note that when testing $H_0^j: \beta^0_j = 0$ under the case where
$n > p$, we have the classical Rao's test statistic, defined by
\begin{equation}
  \label{eq:rao-test}
  T_j^{\text{Rao}} = \sqrt{n} \, \hat{\bI}_{j \mid -j}^{-1/2} \,
  \nabla_{j}\ell( 0 , \hat{\bm\beta}_{-j}) \, ,
\end{equation}
where
$\nabla_{j}\ell( 0 , \hat{\bm\beta}_{-j}) \egaldef
\left.{\nabla_{\beta_j}\ell(\beta_j , \hat{\bm\beta}_{-j})}\right|_{\beta_j=0}$
is the Fisher score.
Here
$\hat{\bm\beta}_{-j} \egaldef \argmin_{\bm\beta_{-j} \in \bbR^{p-1}}
\ell(\beta_j , \bm\beta_{-j})$ is the constrained maximum-likelihood estimator
of $\bm\beta_{-j}$ with fixed $\beta_j$, and $\hat{\bI}_{j \mid -j}$ is a
consistent estimator of the partial Fisher information $\bI_{j \mid -j}$ \ .
The appearance of the term $\hat{\bI}_{j \mid -j}^{-1/2}$ is due to the fact
that under the null hypothesis $H_0^j$, we have, by \citet[Chapter
9]{cox_theoretical_1979}, and by \citet{rao_large_1948},
\begin{equation*}
  \sqrt{n} \nabla_{j}\ell(0 , \hat{\bm\beta}_{-j}) \xrightarrow[n \to \infty]{(d)}
  \cN(0, \bI_{j \mid -j}) 
  \, ,
\end{equation*}
which makes the asymptotic distribution of $T_j^{\text{Rao}}$ standard normal.
However, in the high-dimension case, where $n < p$, we do not reach
this convergence in distribution.
To see this, consider the Taylor expansion of the Fisher score of variable $j$
around any given estimator $\widetilde{\bm\beta}_{-j}$ of the true
$\bm\beta^0_{-j}$:
\begin{equation}
  \label{eq:taylor-score}
  \nabla_{j}\ell( 0 , \widetilde{\bm\beta}_{-j}) 
  = \nabla_{j}\ell( 0 , \bm\beta_{-j}^0) 
  + \nabla^2_{j , -j} \ell( 0, \bm\beta_{-j}^0)(\widetilde{\bm\beta}_{-j} -
  \bm\beta_{-j}^0) + \cO \left((\widetilde{\bm\beta}_{-j} -
  \bm\beta_{-j}^0)^2\right)
\end{equation}
On the right-hand side, the first term converges weakly to a normal
distribution due to the Central Limit Theorem, the remainder term becomes
negligible using $\ell_1$ regularization to induce sparsity, but the second
term does not, due to estimation bias and sparsity effect of
$\widetilde{\bm\beta}_{-j}$ \citep{fu_asymptotics_2000}.
\paragraph{Adapting distillation operation for sparse logistic regression}
Fortunately, Eq.~\eqref{eq:taylor-score} suggests that for each variable $j$,
we can \emph{debias} the Fisher score by correcting the impact of other terms.
In particular, for each variable $j$, we replace the Fisher score by
\begin{equation}
  \nabla_{j}\ell(0, \bm\beta_{-j}) -
  \bI_{j, -j}\bI_{-j, -j}^{-1} \nabla_{-j}\ell(0, \bm\beta_{-j}) \
  .
\end{equation}
The inversion of the large matrix $\bI_{-j, -j} \in
\bbR^{(p-1)\times(p-1)}$ is computationally prohibitive, but we can
estimate the term $\bI_{j, -j}\bI_{-j, -j}^{-1}$ straightforwardly by
solving
\begin{equation}
  \label{eq:estimate-w}
  \hat{\bw}^j 
  = \argmin_{\bw \in \bbR^{p-1}} \dfrac{1}{2n} \sum_{i=1}^n 
  \left[ \nabla^2_{j, -j} \ell_i(\hat{\bm\beta})
   - \bw^T \nabla^2_{-j,-j} \ell_i(\hat{\bm\beta}) \right]^2 + \lambda \norm{\bw}_1, 
\end{equation}
for each variable $j$, where $\hat{\bm\beta}$ is given with Eq.~\eqref{eq:mle}.
Moreover, since we have the closed-form of the derivatives of the logistic loss
$\ell(\hat{\bm\beta})$, a simple derivation from Eq.~\eqref{eq:estimate-w}
suggests the following $x_j$-distillation operation, \emph{adapted for logistic
  regression}:
\begin{equation}
  \label{eq:scaled-betadx} \betadx 
  = \argmin_{\bm\beta \in \mathbb{R}^{p-1}} \dfrac{1}{n} \sum_{i=1}^n
  g''(\bX_{i, *}\hat{\bm\beta}) (\bX_{i,j} - \bX_{i, -j}\bm\beta)^2 + \lambda_{dx}\norm{\bm\beta}_1,
\end{equation}
where the extra term (second-order derivative of the sigmoid function)
$g''(\bX_{i, *}\hat{\bm\beta}) = \frac{\exp{(\bX_{i, *}\hat{\bm\beta})}}{[1 +
  \exp{(\bX_{i, *}\hat{\bm\beta})}]^2}$ appears from differentiating twice the
loss function $\ell(\hat{\bm\beta})$, and $\hat{\bm\beta} = \betamle$ is
defined in Eq.~\eqref{eq:mle}.
On the other hand, we can obtain $\betady_{j}$ from $\hat{\bm\beta}$ by simply
omitting the $j$-th coefficient from it, \ie
\begin{equation*}
\betady \egaldef (\hat{\beta}_1, \hat{\beta}_2, \dots,
  \hat{\beta}_{j-1}, \hat{\beta}_{j+1}, \dots, \hat{\beta}_p) \ .
\end{equation*}

Finally, the equation for decorrelated test score, adapted from both
Eq.~\eqref{eq:test-score-dcrt} and \eqref{eq:rao-test}, reads
\begin{equation}
  \label{eq:decorr-score} T_j^{\text{decorr}} =
- \frac{1}{\sqrt{n}} \ \hat{\bI}_{j \mid -j}^{-1/2} \ \sum_{i=1}^n \left[y_i - g(\bX_{i,-j}\betady)\right] \left[\bX_{i,j}
-\bX_{i,-j} \betadx 
\right],
\end{equation}
where the formula for empirical partial Fisher information is
  $\hat{\bI}_{j \mid -j} = n^{-1} \sum_{i=1}^n
  g''(\bX_{i, *}\hat{\bm\beta}) (\bX_{i,j} - \bX_{i, -j} \ \betadx
  ) \ \bX_{i,j}.$
A summary of the full procedure, which we call CRT-logit, can be found in
Algorithm~\ref{alg:CRT-logit}.
Notice that the runtime of CRT-logit is the same as dCRT, which means in
general slower than KO and HRT.
To speedup inference time, we introduce a variable-screening step that
eliminates potentially unimportant variables before distillation, similar to
dCRT.
We provide empirical benchmark of runtime of each method in
Section~\ref{ssec:runtime}.
\paragraph{Setting $\ell_1$-regularization parameter $\lambda$ and $\lambda_{dx}$}
In general, we advise to use cross-validation for obtaining $\betamle$ in
Eq.~\eqref{eq:mle} and for $\bX_{*, j}$-distillation operator, as defined by
Eq.~\eqref{eq:scaled-betadx}.
This is inline with the theoretical argument for dCRT \cite[Lemma 1 and Theorem
3]{liu_fast_2020}.
However, we also observe empirically that choosing the $\ell_1$-regularization
parameters of the distillation step can strongly affect how variables are
selected by CRT-logit.
We provide more details Appendix~\ref{ssec:setting-l1}, and leave further
theoretical investigations of this phenomenon for future work.
\begin{algorithm}[h]
  \small
  \SetAlgoLined
  {\textbf{INPUT} design matrix $\bX \in \bbR^{n \times p}$, reponses $\by \in \bbR^n$} \\
{\textbf{OUTPUT} vector of p-values $\{p_j\}_{j=1}^p$};  \\
$\betamle \gets \texttt{solve\_sparse\_logistic\_cv}(\bX, \by
)$ \tcp{Using
  Eq.~\eqref{eq:mle} 
}
$\hat{\cS}^{\text{SCREENING}} \gets \{j: j \in [p], \hat{\beta}_j \neq 0\}$  \\
\For{$j \in \hat{\cS}^{\text{SCREENING}}$}{
  $\betadx
  \gets \texttt{solve\_scaled\_lasso\_cv}(\bX_{*,
      j}, \bX_{*,-j}
    )$
    \tcp{Using Eq.~\eqref{eq:scaled-betadx}
    }
  $\betady
    \gets (\hat{\beta}_1, \hat{\beta}_2, \dots,
    \hat{\beta}_{j-1}, \hat{\beta}_{j+1}, \dots, \hat{\beta}_p)$
    
  $T_j^{\text{decorr}} \gets \texttt{decorrelated\_test\_score}(j, \bX,\by,
    \betadx
    , \betady
    )$ \tcp{Using Eq~\eqref{eq:decorr-score}}
  $\hatP_j \gets 2[1 - \abs{\Phi\left( T_j^{\text{decorr}} \right)}]$
  }
  \For{$j \notin \hat{\cS}^{\text{SCREENING}}$}{
    $\hatP_j = 1$
    }
\caption{CRT-logit}
\label{alg:CRT-logit}
\end{algorithm}
\paragraph{Asymptotic analysis of the Decorrelated Test Statistic}
We now provide theoretical analysis of CRT-logit in large-sample regime.
All the proofs can be found in Appendix~\ref{sec:proofs}.
Without writing it explicitly, in our analysis, we consider $p = p(n)$, and the
following assumption.
\begin{assn}[Regularity conditions]
  \label{astn:regularity-glm}
  Assume that
  \begin{enumerate}[label=(A\arabic*)]
  \item $\lambda_{\min}(\bI) \geq \kappa^2$ for some constant $\kappa > 0.$
  \item Sparsity of $\betatrue$ and $\wjtrue$, with $\wjtrue$ the ground truth
    weights for the distillation of $\bx_j$ in Eq.~\eqref{eq:scaled-betadx}:
    $\abs{\cS} =s^*$ and $\norm{\wjtrue}_0 = s'$ with
    $s^* = o \bigl( n^{1/2}/ \log (p) \bigr)$ and
    $s' = o \bigl( n^{1/2} / \log (p) \bigr)$.    
  \item For all $i \in [n]$, $\bX_{i, *}$ and $(-y_i + g'(\bX_{i,*}\beta))$
    are sub-exponential random variables, and
    $\abs{\bX_{i, -j}\wjtrue} \leq K$ almost surely, for some constant $K$.
  \end{enumerate}
\end{assn}
We then have the following result, that states that the asymptotic distribution of
the decorrelated test scores is standard normal.
\begin{theo}
  \label{thm:clt-crt-logit}
  Let $j \in [p]$, and let $T_j^{\text{decorr}}$ be defined as in
  Eq.~\eqref{eq:decorr-score}, with
  $\lambda \asymp \lambda_{dx} \asymp \sqrt{n^{-1}\log (p) }$.
  Then, if Assumption~\ref{astn:regularity-glm} holds true, under the null
  hypothesis $\cH_0^j: \beta^0_j = 0$, we have
  \begin{align*}
    \forall t \in \bbR \, , \qquad 
    \lim_{n\to\infty}\abs{\bbP_{\beta^0}(T_j^{\text{decorr}} \leq t) - \Phi(t)}
    &= 0 \ ,
  \end{align*}
where $\Phi(\cdot)$ is the CDF of the standard Gaussian distribution.
Moreover, for each $j \in [p]$, if we define
$\hatP_j \egaldef 2\left[1 - \Phi\left( T_j \right) \right]$ , \ie
$\hatP_j$ is the output of Algorithm~\ref{alg:CRT-logit}, then, under the
null hypothesis $\cH_0^j: \beta^0_j = 0$, we have
\begin{equation*}
  \limsup_{n \to \infty} \bbP_{\beta^0}(\hatP_j \leq t) \leq t
  \quad \text{for all } t \in [0, 1] \, ,
\end{equation*}
that is, the p-values output by Algorithm~\ref{alg:CRT-logit} are valid
asymptotically.
\end{theo}
\paragraph{FDR control with CRT-logit}
We now state the second main result, which establishes that the FDR of the test
is controlled when using Benjamini-Yekutieli procedure
\citep{benjamini_control_2001} with the p-values output from
\Cref{alg:CRT-logit}, assuming that the number of tests $p$ is fixed.
\begin{theo}
  \label{thm:asymp-fdr-control-crt-logit}
  Under Assumptions~\ref{astn:regularity-glm} and logistic model defined in
  Eq~\eqref{eq:logit}, with
  $\lambda \asymp \lambda_{dx} \asymp \sqrt{n^{-1}\log (p) }$, assume
  $n^{-1/2}(s' \vee s^*)\log (p) = o(1)$, and assume the number of tests $p$ is
  fixed.
  Let $\alpha \in (0,1)$ and $\Sbycrt$ be given by applying following the
  Benjamini-Yekutieli FDR-controlling procedure to the CRT-logit p-values
  $\{\hatP_j\}_{j \in [p]}$, output from Algo.\ref{alg:CRT-logit}.
  Then, we have
  \begin{equation*}
    \limsup_{n \to \infty} \bbE\left[\dfrac{\card (\Sbycrt \cap
          \cS^c)}{\card(\Sbycrt) \vee 1} \right] \leq \alpha
          \, . 
  \end{equation*}
\end{theo}
\begin{rk}
  Assumption~\ref{astn:regularity-glm} is also assumed in
  \cite{ning_general_2017,van_de_geer_asymptotically_2014}, which also provide
  a detailed discussion of this regularity assumption in generalized linear
  models.
  This assumption, in turn, is built on the regularity assumption in the
  classic work of \citet[Chapter 9]{cox_theoretical_1979} to establish asymptotic
  normality of Rao's test statistic.
  Theorem~\ref{thm:clt-crt-logit} is an adaptation of
  \citet[Theorem~3.1]{ning_general_2017}, specialized for the case of sparse
  logistic regression and the p-values output from CRT-logit.
\end{rk}
%

\section{Empirical Results}
\label{sec:empirical}
%
We provide benchmarks of the proposed CRT-logit algorithm along with most other
methods mentioned in the introduction, in particular: model-X Knockoff (KO)
\citep{candes_panning_2018}, Debiased Lasso (dLasso)
\citep{zhang_confidence_2014,javanmard_confidence_2014}, original CRT with 1000
samplings \citep{candes_panning_2018}, Holdout Randomization Test with 5000
samplings \citep{tansey_holdout_2018}, and Lasso-Distillation CRT (dCRT)
\citep{liu_fast_2020}.
We did not include SLOE \citep{yadlowsky_sloe_2021} and CPT
\citep{berrett_conditional_2019}, as the provided open-source implementation are
particularly unstable and do not fit in the sparse-regression setting (for
SLOE), or implementation are not available (for CPT).
\begin{rk}
  \label{rk:noise-term}
  As a slight caveat, in the simulated and semi-realistic experiment sections
  (Sections~\ref{ssec:qqplot}, \ref{ssec:simu-mildly-high} and
  \ref{ssec:simu-hcp}), we introduce an additional noise term to the logistic
  relationship of Eq.~\eqref{eq:logit}:
  \begin{equation}
    \label{eq:logit-noise}
    \bbP(y_i = 1 \mid \bx_i) = g(\bx_i^T\bm\beta^0 + \sigma\xi_i) \ , 
  \end{equation}
  where $\xi_i \sim \cN(0, 1)$ is a Gaussian noise and $\sigma > 0$ the noise
  magnitude.
  The formula in Eq.~\eqref{eq:logit-noise} has been used in previous works,
  \eg \citet{bzdok_semi-supervised_2015}.
  There is a clear justification to this: in most of the applications we consider,
  data are collected with measurement errors.
  In the case of brain-imaging, for example, recording the brain signal of the
  human subjects by scanners often includes noise caused either from the
  machine, or from the movement of the subjects, as elaborated by
  \citet{lindquist_statistical_2008}.
  Moreover, in general, this setting corresponds to a model mis-specification,
  which the CRT-logit is robust to under \Cref{astn:regularity-glm}, following
  the arguments as in \citet[Section 5]{ning_general_2017}.
\end{rk}
\begin{rk}
  We use Benjamini-Hochberg step-up procedure
  \citep{benjamini_controlling_1995} to control FDR with the p-values in all
  the empirical experiments in Section~\ref{ssec:simu-mildly-high} and
  App.~\ref{ssec:gwas}, as we observe empirically the FDR is usually controlled
  with this procedure, without compromising power with the conservative BY
  bound.
\end{rk}
\subsection{Effectiveness of the decorrelation step}
\label{ssec:qqplot}
To show how decorrelating the test statistics helps, we set up a simulation
with matrix $\bX$ of $p=400$ features and vary the number of samples
$n \in \{200,400,800,4000\}$.
The binary response vector $\by$ is created following
Eq.~\eqref{eq:logit-noise}, and
the design matrix $\bX$ is sampled from a multivariate normal distribution with
zero mean, while the covariance matrix $\bm\Sigma \in \bbR^{p \times p}$ is a
symmetric Toeplitz matrix, where the parameter $\rho \in (0, 1)$ controls the
correlation between neighboring features: correlation decreases quickly when
the distance between feature indices increases.
The true signal $\bm\beta^0$ is picked with a sparsity parameter
$\kappa = s^* / p$ that controls the proportion of non-zero elements with
magnitude 2.0, \ie $\beta_j = 2.0$ for all $j \in \cS$.
For the specific purpose of this experiment, non-zero indices of $\cS$ are
kept fixed.
The noise $\bm\xi$ is \iid normal $\cN(0, \textbf{Id}_n)$ with magnitude
$\sigma = \norm{\bX\bm\beta^0}_2 / (\sqrt{n} \ \text{SNR})$, controlled by the
SNR parameter.
In short, the three main parameters controlling this simulation are correlation
$\rho$, sparsity degree $\kappa$ and signal-to-noise ratio SNR.
We generate randomly 1000 datasets, and run dCRT and CRT-logit algorithm to
obtain one sample of test statistics $\{T_j\}_{j=1}^p$ and
$\{T_j^{\text{decorr}}\}_{j=1}^p$.
Then, we pick 1000 samples of one null test statistic $T_j$ and
$T_j^{\text{decorr}}$, defined in Eq.~\eqref{eq:test-score-dcrt}
and~\eqref{eq:decorr-score}, respectively, and plot the qq-plot of their
empirical quantile versus the standard normal quantile.
\begin{figure}[h]
  \centering
  \begin{subfigure}[c]{0.22\linewidth}
    \centering
    \includegraphics[width=0.8\textwidth]{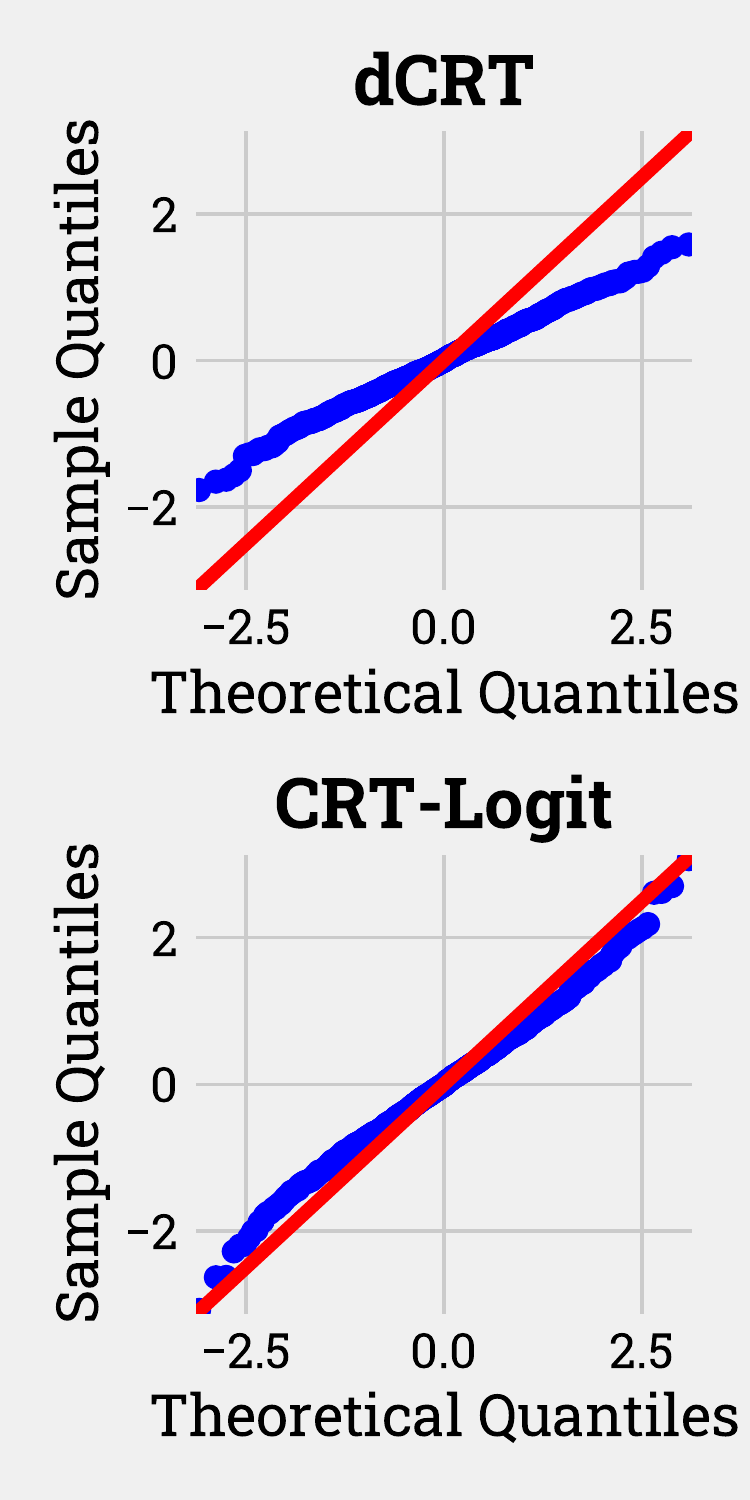}
    \caption{$\mathbf{n=200}$}
  \end{subfigure}
  \begin{subfigure}[c]{0.22\linewidth}
    \centering
    \includegraphics[width=0.8\textwidth]{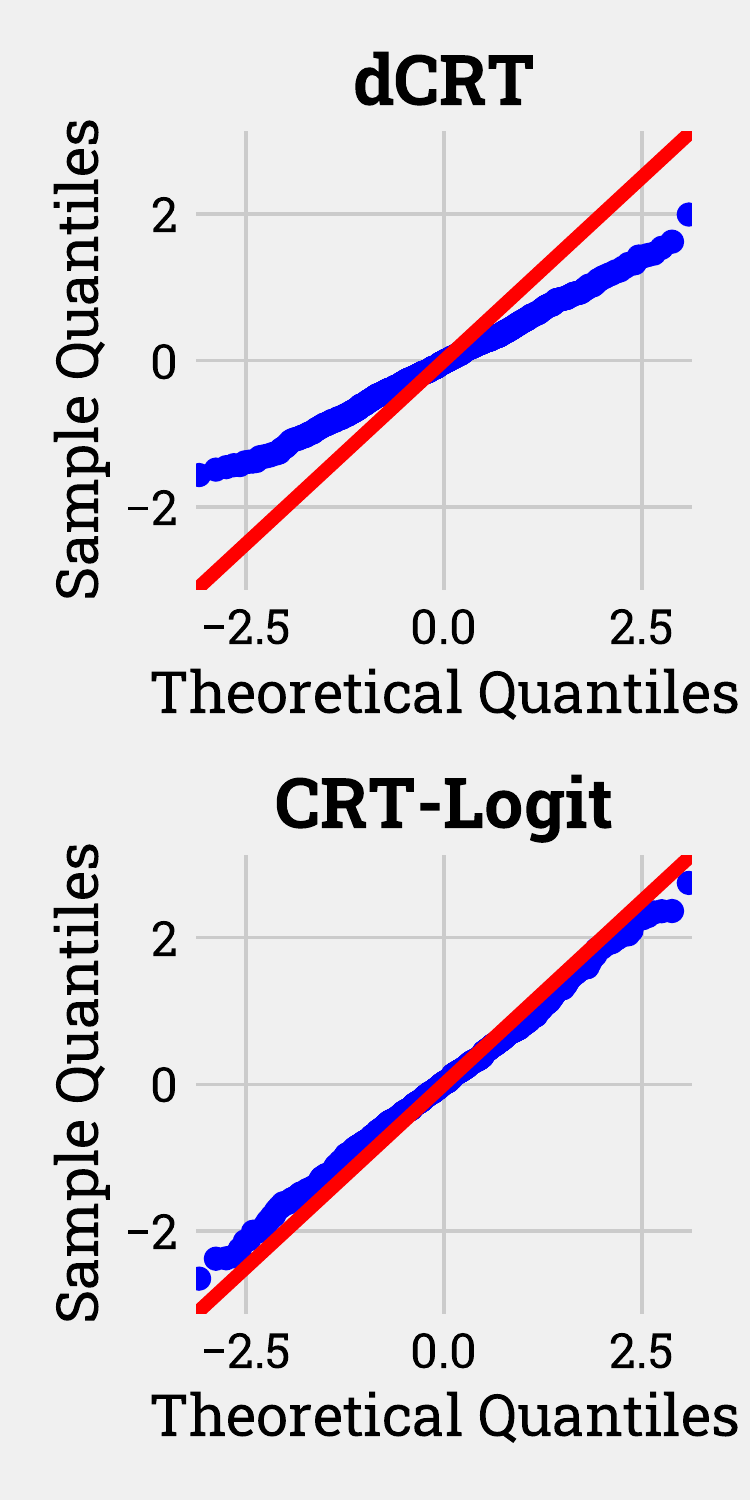}
    \caption{$\mathbf{n=400}$}
  \end{subfigure}
  \begin{subfigure}[c]{0.22\linewidth}
    \centering
    \includegraphics[width=0.8\textwidth]{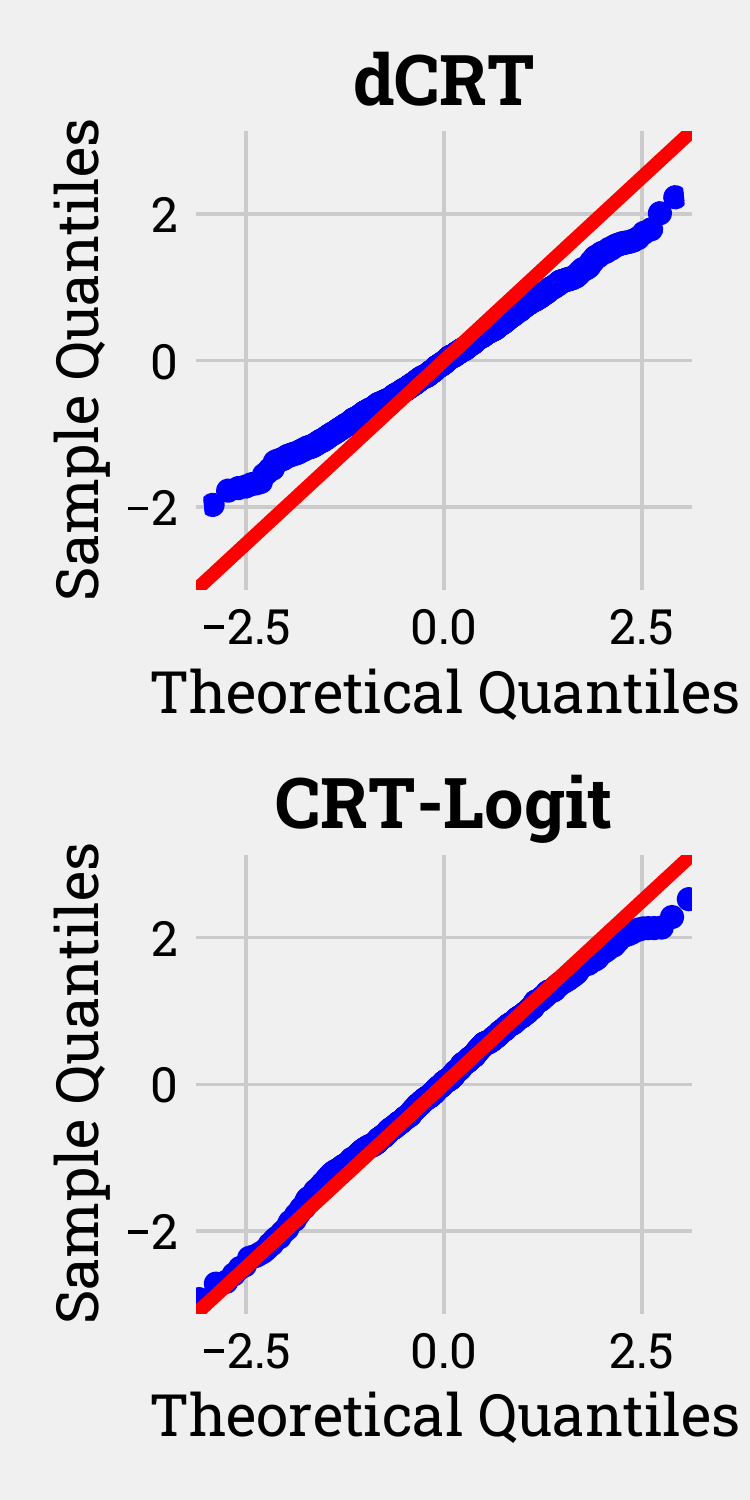}
    \caption{$\mathbf{n=800}$}
  \end{subfigure}
  \begin{subfigure}[c]{0.22\linewidth}
    \centering
    \includegraphics[width=0.8\textwidth]{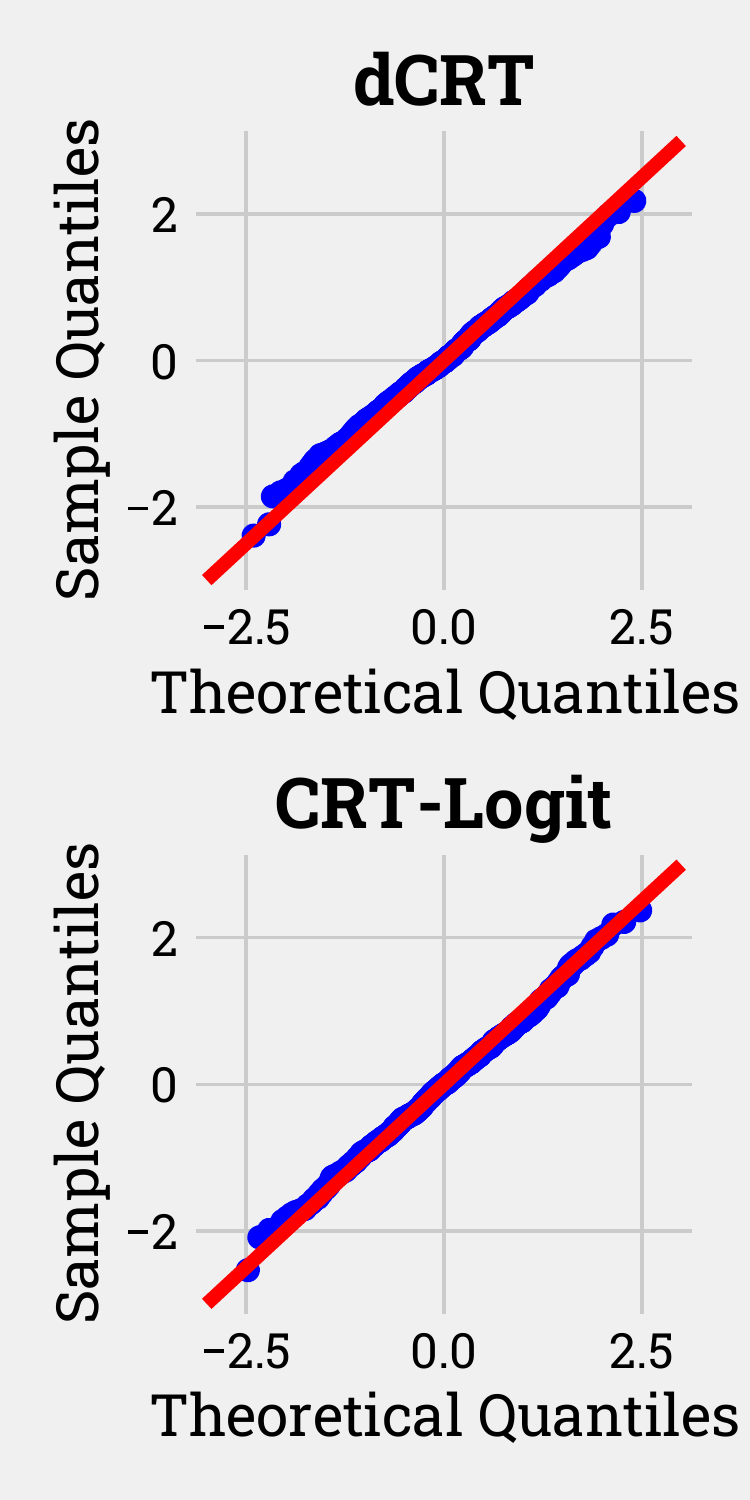}
    \caption{$\mathbf{n=4000}$}
  \end{subfigure}
  \caption{\textbf{QQ-Plot for one null CRT statistic for logistic regression,
      with varying number of samples and a fixed number of variables
      $p=400$}.
    %
    The theoretical quantiles are obtained from a standard Gaussian
    distribution.
    The decorrelation step makes the empirical null distribution of the null
    statistics much closer to standard Gaussian.
    Parameters: $\text{SNR}=3.0, \rho=0.4, \text{ sparsity}=0.06$. \emph{Upper
      row: Distilled-CRT statistic defined by Eq.~\eqref{eq:test-score-dcrt}.
      Bottom row: CRT-logit, with decorreleated test score defined by
      Eq.~\eqref{eq:decorr-score} \textbf{(ours)}.}  }
  \label{fig:qqplot}
\end{figure}
From the results in Figure~\ref{fig:qqplot}, we observe that the empirical null
distribution of the test statistic is \emph{much closer to a standard normal
when adding the decorrelation step.}
In particular, when the sample size $n$ increases to 400, the decorrelated test
statistic has empirical quantiles almost inline with the theoretical quantiles
of the standard normal distribution, while dCRT test score strays away from the
45-degree line.
Again, we emphasize that the normality of $T_j$ is crucial for the p-values
calculation.
This outlines the importance of the decorrelating step on $T_j$.
\subsection{High-dimensional scenario with varying simulation parameters}
\label{ssec:simu-mildly-high}
To see how each algorithm performs under different settings, we follow the same
simulation scenario as in \Cref{ssec:qqplot}, but vary each of the three
simulation parameters, while keeping the others unchanged at default value of
SNR $=2.0$, $\rho=0.5$, $\kappa=0.04$.
We target a control of FDR at level $0.1$, using Benjamini-Hochberg procedure.  
Results in Figure~\ref{fig:fdr-power-mild} show that CRT-logit is the most
powerful method while still controlling the FDR.
Moreover, in the presence of higher correlations between nearby variables
($\rho > 0.6$), other methods suffer a drop in average power, but this is not
as severe for CRT-logit.
The original CRT, in general, is conservative.
We believe that this is due to using only $B=500$ samplings to generate
empirical p-values for the two methods, due to prohibitive average runtime of
the algorithm with larger $B$ (which we provide in Section~\ref{ssec:runtime}).
For HRT, the conservativeness is expected, due to the usage of only half of the
sample for test-statistics calculation -- even though the number of samplings is
bigger than original CRT ($B=5000$).
We note that, perhaps surprisingly, the debiased lasso (\texttt{cdlasso}) is
the most conservative. It controls FDR well in all settings.
This might be due to the fact that \texttt{dlasso} also relies on the choice of
the $\ell_1$-regularization $\lambda$ in the nodewise Lasso operation, similar
to the $\bX_{*, j}$-distillation of dCRT, as noted in
Section~\ref{sec:intro-dcrt}.
What makes the difference is that instead of using cross-validation for
setting $\lambda$ for each variable $j$,
a \emph{fixed} value of $\lambda=10^{-2}\lambda_{max}$ is used in the
implementation of \texttt{dlasso}.
We strongly suspect this fixed value is not optimal, which makes the procedure
powerless.

\begin{figure}[h]
  \centering
  \includegraphics[width=0.5\columnwidth]{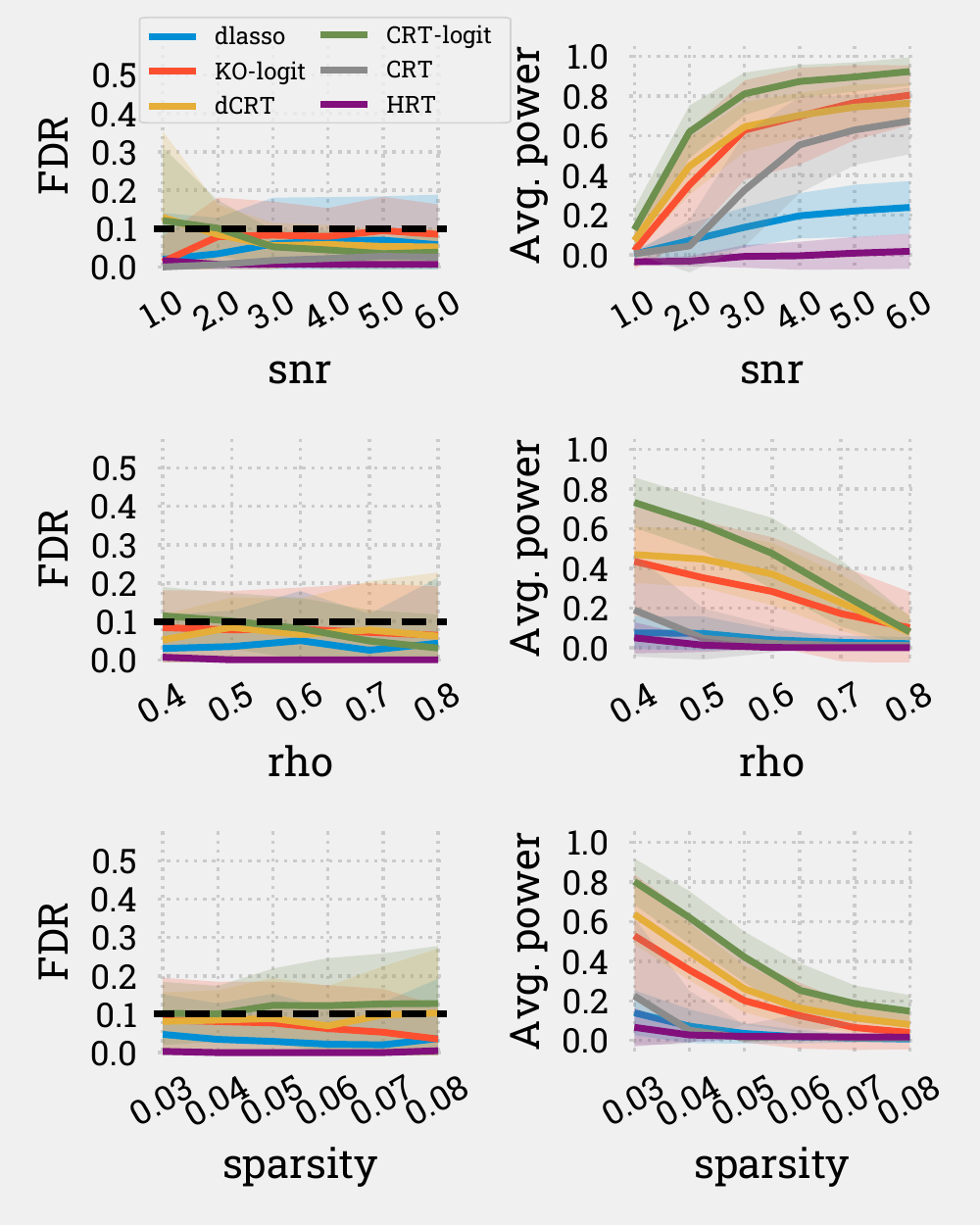}
  \caption{\textbf{FDR/Average Power of 100 runs of simulations across varying
      parameters in high-dimensional settings}. Default parameter:
    $n=400, p=600, \text{SNR}=2.0, \rho=0.5, \kappa=0.04$. FDR is controlled at
    level $\alpha=0.1$. Methods: Debiased Lasso (\texttt{dlasso}), model-X
    Knockoff (\texttt{KO-logit}), original CRT (\texttt{CRT}), HRT
    (\texttt{HRT}), dCRT (\texttt{dCRT}), and our version of CRT (dark green
    line -- \texttt{CRT-logit}).}
  \label{fig:fdr-power-mild}
\end{figure}

\subsection{Application: large-scale analysis on brain-imaging dataset}
\label{ssec:simu-hcp}
\paragraph{Description}
The Human Connectome Project dataset (HCP) is a collection of brain imaging data
on healthy young adult subjects with age ranging from 22 to 35.
More specifically, the input $\bX$ is a set of 2mm statistical maps of
400 subjects across 8 cognitive tasks.
These are called z-maps, as the data follow a standard normal distribution
under the null hypothesis.
Each task in turn features 2 different contrasts, which effectively form
binary responses $\by \in \{0, 1\}^n$.
We propose to fit $\by$ through distributed brain signals and identify relevant
brain locations.
The setting is high-dimensional with $n = 800$ samples, corresponding to 400
subjects, while the total number of variables is $p \approx 200,000$ brain
voxels.
Following \cite{nguyen_ecko_2019,chevalier_decoding_2021}, we use a
hierarchical clustering scheme to group the variables into $C = 1000$ spatially
connected clusters.
We provide details of the pre-processing step in
Appendix~\ref{sec:details-experiment}.
\paragraph{Creating semi-realistic ground-truth and response labels}
Since there is no ground truth for this dataset, we create synthetic true
signals by fitting the data $\bX$ and response $\by$ with an $\ell_1$-penalized
logistic classifier.
In other words, the estimator $\betalogreg$ will serve as true
regression coefficients for each task.
Then, to avoid bias in simulating label $\hat{\by}$, the z-maps matrix $\bX $
of one task are used in conjunction with the discriminative pattern map
$\betalogreg$ from the next task in the following order: \texttt{relational},
\texttt{gambling}, \texttt{emotion}, \texttt{social}.
For instance, we use $\betalogreg$ of \texttt{gambling} with z-maps data matrix
of \texttt{relational}, \ie for all $i = 1, \dots, n$, given
$\bx_{i, \texttt{relational}}$,
\begin{equation} \label{eq.simu-hcp.gener-yhat}
  \hat{y}_{i} \sim
  \text{Bern}\left\{g(\bx_{i, \texttt{relational}}^{\top} \ \betalogreg_{\texttt{gambling}}
    + \sigma\xi_i)\right\},
\end{equation}
where $\text{Bern}(a)$ is a Bernoulli probability mass function that takes a
value 1 with probability $a$, $\sigma$ is a noise magnitude and $\xi_i$ is a
standard normal noise.
Finally, we apply all inference algorithms on the semi-synthetic data
$(\bX, \hat{\by})$, and we evaluate their performance using the ground-truth
$\hat{\bm\beta}^{\text{logreg}}$.
This simulation setting is similar to \cite{chevalier_decoding_2021}, except
that here we consider a classification and not a regression problem.
It allows us to calculate the False Discovery Rate and average power with
multiple runs of the inference procedure (across tasks).
\begin{rk}
  The \iid assumption is formally violated in this experiment, where for each
  subject we analyze two sample image that are not
  independent.
  Yet, this remains a short-range correlation structure, and is thus not a strong
  challenge to the \iid assumption.
\end{rk}
\begin{figure}[h]
  \centering
  \includegraphics[width=0.8\textwidth]{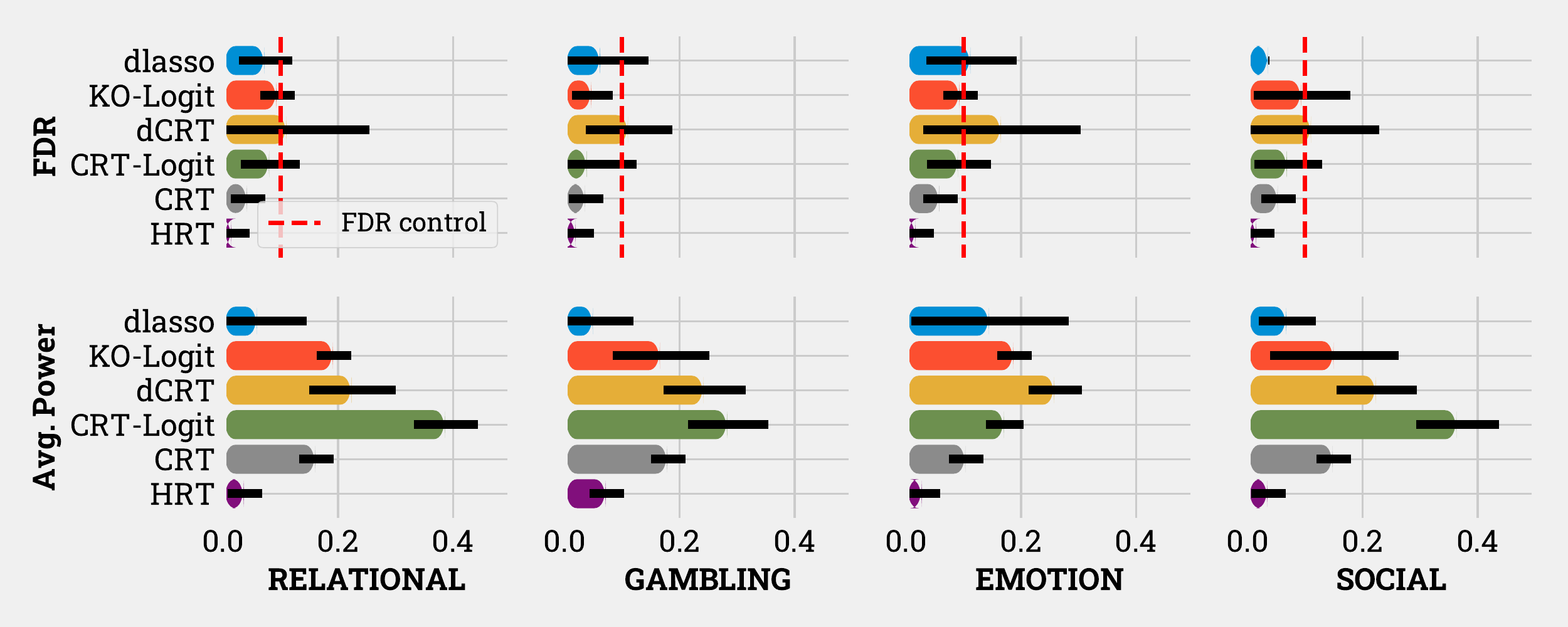}
  \caption{\textbf{FDR/Average Power of 50 runs of semi-realistic experiments
      on four tasks of Human Connectome Project dataset.}
    Parameters: $n=800$ (taken from 400 subjects), $\text{SNR}=2.0$.
    Methods (clustering versions): Debiased Lasso (\texttt{cdlasso}), model-X
    Knockoff (\texttt{cKO-logit}), original CRT (\texttt{CRT}), HRT
    (\texttt{HRT}), dCRT (\texttt{dCRT}), and our version of CRT (dark
    green line -- \texttt{CRT-logit}).}
  \label{fig:fdr-power-hcp}  
\end{figure}
\paragraph{Results}
Results in Figure~\ref{fig:fdr-power-hcp} show that CRT-logit
achieves a better recovery compared to KO or original CRT/dCRT/HRT, which
results in higher statistical power.
This gain comes with a good control of the FDR under desired level
$\alpha=0.1$.
On a related note, the only analysis where dCRT makes more discoveries than
CRT-Logit is in \texttt{emotion} task, but at the cost of failing to control
FDR at nominal level.

\subsection{Application:  genome-wide association study with
  Human Brain Cancer Dataset}
\label{ssec:gwas}
\paragraph{Description}
The last in our benchmark is a Genome-wide Association Study (GWAS) on the The
Cancer Genome Atlas (TCGA) dataset
\citep{weinstein_2013_cancer,vasaikar_linkedomics_2018}.
We choose to analyze the Glioma cohort, which consists of $n=1026$ patients
across a wide age range, diagnosed with this type of brain tumor, with a total
of $p=24776$ genes in the data matrix, recorded as copy number variations
(CNVs) at the gene level in log ratio format.
As with the brain-imaging inference in Section~\ref{ssec:simu-hcp}, we use
clustering to reduce the dimension to $C=1000$ clusters.
For the response, a long-term survivor (LTS) is defined as a patient who
survived more than five years after diagnosis and would be labeled $y=0$, and
any patient who died within five years would be a short-term survivor (STS),
labeled $y=1$.
The objective is to identify significant genes that contribute to
classification of the LTS/STS status.
Similar to the Human Connectome Project dataset, there is no real ground-truth
for the TCGA Glioma.
However, we have the list of mutations and the frequency of those detected in
the diagnosed patients.
We therefore select the 1000 most frequent gene mutations that appeared in this
list, \ie the ground truth list consists of 1000 genes (variables).
\begin{table}[h]
  \centering
  \small
  \caption{\textbf{List of detected genes associated with Glioma Cancer from
      the TCGA dataset $n=1026, p=24776$}. Empty line (---) signifies no
    detection. Methods listed in the table are the clustering version.
    Commonly detected genes between methods are put in bold text. Most detected
    genes are listed in the mutant list database that can be found in the
    recorded patients \citep{vasaikar_linkedomics_2018}.}
  \vskip 0.15in
\begin{tabular}{ll}
  \toprule
  \textbf{Methods}
         & \textbf{Detected Genes} \\
  \midrule
  dLasso & --- \\
  KO     & \textbf{ABCC10}, \textbf{ANK3}, CDH23, PTEN, \textbf{SPEN}, \textbf{SVIL}, ZMIZ1 \\
  dCRT  & \textbf{ANK3}, \textbf{ANKRD30A}, CDH23, PTEN, RET, \textbf{SPEN},ZMIZ1 \\
  CRT-logit & \textbf{ABCC10}, \textbf{ANKRD30A}, BCOR, EPHA3, PPL, SPAG17, \textbf{SPEN}, \textbf{SVIL}, USP9X\\
  Original CRT & \textbf{ABCC10}, BCOR, EPHA3, \textbf{SPEN}, \textbf{SVIL} \\
  HRT & \textbf{ABCC10}, \textbf{SPEN} \\
  \bottomrule
\end{tabular}
\vskip -0.1in
\label{table:detected-genes}  
\end{table}
\paragraph{Result }
The result from Table \ref{table:detected-genes} shows that CRT-logit finds
the largest number of genes.
Moreover, most of selected genes in this table are detected in the list of mutated
genes found on recorded patients.
Some genes are detected by all the benchmarked methods, most prominently
\texttt{SPEN}, which is found on over 10 \% of patients in the cohort.
Furthermore, this gene is known to be associated not only with  brain
cancer, but also with other types of cancer in The Human Protein Atlas project
\citep{legare_estrogen_2015}.
Note that, in the absence of a ground-truth, this does not guarantee all genes
found are associated with glioma, but this experiment demonstrates the
capability of CRT-logit in GWAS studies.

\subsection{Average runtime of benchmarked methods}
\label{ssec:runtime}
\begin{table}[h]
  \small
  \centering
  \caption{\textbf{Average runtime of benchmarked methods for one simulation
      (in seconds)}. Standard error is reported in parentheses.}
  \label{table:runtime}  
  \vskip 0.15in
\begin{tabular}{p{6cm}cc}
  \toprule
  \textbf{Methods} & \textbf{Simulated (Sec.~\ref{ssec:simu-mildly-high})}
 & \textbf{HCP-semi-real (Sec.~\ref{ssec:simu-hcp})}  \\
  \midrule
  Debiased Lasso
  \citep{zhang_confidence_2014,van_de_geer_asymptotically_2014,javanmard_confidence_2014}
                   & $61.83 \, (5.2)$ & $154.27 \, (8.79)$ \\
  Knockoff Filter \citep{barber_controlling_2015,candes_panning_2018}
                   & $1.62 \, (0.02)$ & $8.12  \, (0.62)$ \\
  CRT (500 samplings) \citep{candes_panning_2018} & $2312.91 \, (38.21)$ & $7069.96 \,
                                                                    (109.09)$ \\
  HRT (5000 samplings) \citep{tansey_holdout_2018} & $14.84 \, (2.01)$ & $52.17 \, (4.11)$\\      
  dCRT[screening=True] \citep{liu_fast_2020} & $16.83 \, (1.89)$ & $65.18 \, (3.91)$\\
  dCRT[screening=False] \citep{liu_fast_2020} & $370.12 \, (8.18)$ & $962.40 \, (20.63)$ \\  
  \textbf{CRT-logit[screening=True] (this work)} & $\textbf{14.16} \, \textbf{(0.35)}$ & $\textbf{61.26} \, \textbf{(3.55)}$ \\
  \textbf{CRT-logit[screening=False] (this work)} & $\textbf{367.91} \, \textbf{(4.11)}$ & $\textbf{983.78} \,  \textbf{(17.26)}$\\
  \bottomrule
\end{tabular}
\end{table}
Besides statistical performance, it is equally important to assess the
computational cost of inference procedures.
The average runtime in Table~\ref{table:runtime} from the two experiments shows
that the original CRT is not suitable for large-scale inference: it is over
2000 times slower than the fastest method (Knockoff Filter), and over 150 times
slower than dCRT/CRT-logit.
The empirical runtime also confirms the effectiveness of the screening step before
doing distillation/decorrelation of the test-statistics: the step makes
CRT-logit and dCRT 20 times faster than without.
On a related note, although in theory Debiased Lasso, dCRT and CRT-logit (both
without screening) share the same runtime complexity, the latter two are slower
due to the use of cross-validation to estimate the sparsity hyperparameter
$\lambda$ and $\lambda_{dx}$ (detailed in Section~\ref{sec:algo}).
%
\section{Discussion}
\label{sec:crt-discussions}
%
We proposed an adaptation of the Conditional Randomization Test (CRT) for
sparse logistic regression in the high-dimensional regime.
A major improvement of CRT-logit, our proposed algorithm, compared to original
CRT, comes from the decorrelation of test statistics to make their distribution
closer to standard normal.
Indeed, results from synthetic experiments in Figure~\ref{fig:qqplot} show
that in high-dimension (when $0.5 \leq n/p \leq 1.0$), the empirical null
distribution of CRT-logit's test statistic $T^{\text{decorr}}$ is much more
similar to a standard normal compared to the original CRT test statistic.
Moreover, empirical benchmarks in Section~\ref{sec:empirical} demonstrate that
CRT-logit performs better than related statistical inference methods, such as
the Debiased Lasso or Model-X Knockoffs.
In particular, CRT-logit is the most powerful method in our synthetic
experiment with high-dimensional datasets in
Section~\ref{ssec:simu-mildly-high}, while still keeping FDR controlled under
predefined level $\alpha=0.1$.

We note that there exists some limitations to CRT-logit.
The computational cost of CRT-logit, while lower than vanilla CRT, is still
larger than alternative methods such as Knockoff Filter and Holdout
Randomization Test.
Moreover, tuning the $\ell_1-$regularization $\lambda_{dx}$ parameter by
cross-validation, as is often done, can further increase the computational cost
of CRT-logit (and dCRT).

Despite this, our empirical benchmarks on both simulated and real data show
real promises of CRT-logit.
Henceforth, we believe CRT-logit is competitive for practical settings that
involve structured data, such as brain-imaging and genomics applications.

\subsection*{Acknowledgements}
BN, BT and SA acknowledged the support of the French "Agence Nationale de la
Recherche" under the project ANR-17-CE23-0011 (FastBig) and ANR-20-CHIA-0025-01
(KARAIB AI chair). BN was also supported by Chair DSAIDIS of T\'elecom Paris.

\bibliographystyle{apalike}
\bibliography{bibliography}

\newpage
\appendix
\section{Proofs of theoretical results in Section~\ref{sec:algo}}
\label{sec:proofs}

We first present some technical lemmas that are useful for the proof of the
main theorem.
From now on, let $\precsim $ and $\succsim$ denote inequalities with a hidden
constant factor, \ie $x \precsim y$ means that with high probability, there
exists an absolute constant $C > 0$ such that $x \leq Cy$, and vice versa.
As mentioned in the main text, in what follows, without writing it explicitly,
we consider $p = p(n)$.
\begin{lemma}[Lemma E.1, \cite{ning_general_2017}]
  \label{lemma:error-bound-beta}
  Assume Assumption~\ref{astn:regularity-glm-modified}, under the logistic model,
  we have
  \begin{equation*}
    \norm{\hat{\bm\beta} - \bm\beta^0}_1 \precsim s^*\sqrt{\dfrac{\log p}{n}}
    \quad \text{and} \quad
    \norm{\hat{\bm\beta} - \betatrue}_2 \precsim \sqrt{\dfrac{s^* \log p}{n}},
  \end{equation*}
  where $s^* = \norm{\bm\beta^0}_0$.
  In addition, we also have
  \begin{equation*}
    \dfrac{1}{n} \sum_{i=1}^n g''(\bX_{i,*}\bm\beta^0) [\bX_{i, -j}(\hat{\bm\beta} -
    \bm\beta^0)]^2 \precsim \dfrac{s^* \log (p)}{n} \ ,
  \end{equation*}
  where $g(x) = 1 / (1 + \exp(x))$ is the sigmoid function.
\end{lemma}
\begin{lemma}[Lemma E.2 \cite{ning_general_2017}, concentration of the gradient
  and Hessian of the logistic loss function]
  \label{lemma:gradient-hessian}
  Assume Assumption~\ref{astn:regularity-glm-modified} holds, under logistic model, we
  have, with $\bv^* \egaldef (1, -\wjtrue)\in \bbR^p$,
  \begin{align*}
    \norm{\nabla\ell(\bm\beta^0)}_{\infty} &\precsim \sqrt{n^{-1}\log p},
    \quad \text{and}\\
    \norm{\bv^{*\top}\nabla^2\ell(\bm\beta^0) -
    \bbE_{\bm\beta^0}[\bv^{*\top}\nabla^2\ell(\bm\beta^0)]}_{\infty}
                                           &\precsim \sqrt{n^{-1}\log p}.
  \end{align*}
\end{lemma}
\begin{lemma}[Lemma E.3, \cite{ning_general_2017}]
  \label{lemma:error-bound-w}
  Assume Assumption~\ref{astn:regularity-glm-modified} holds, under logistic
  model, we have
  \begin{equation*}
    \norm{\betadx - \wjtrue}_1 \precsim (s' \vee s^*)\sqrt{\dfrac{\log p}{n}},
  \end{equation*}
  where $s^* = \norm{\bm\beta^0}_0$ and $s' = \norm{\wjtrue}_0$.
  In addition, we also have
  \begin{equation*}
    \dfrac{1}{n} \sum_{i=1}^n g''(\bX_{i,*}\hat{\bm\beta}) [\bX_{i, -j}(\betadx -
    \wjtrue)]^2 \precsim \dfrac{(s' \vee s^*)\log (p)}{n}.
  \end{equation*}
\end{lemma}
\begin{lemma}[Lemma E.4, \cite{ning_general_2017}, local smoothness conditions
  on the loss function]
  \label{lemma:smoothness}
  Let $\betanull = (0, \hat{\bm\beta}_{-j}) \in \bbR^p$, where $\hat{\bm\beta}$
  is an estimator of $\betatrue$.
  %
  It holds that
  \begin{align*}
    \bv^{*\top}[\nabla\ell(\bm\beta) - \nabla\ell(\betatrue) -
    \nabla^2\ell(\betatrue)(\bm\beta - \betatrue)]
    &\precsim \dfrac{(s^* \vee s') \log p}{n} \ ,\\
    (\hat{\bv} - \bv^{*})^{\top}[\nabla\ell(\bm\beta) -
    \nabla\ell(\betatrue)] &\precsim \dfrac{(s^* \vee s') \log p}{n} \ .
  \end{align*}
  for both $\bm\beta = \betanull$ and $\bm\beta = \hat{\bm\beta}$.
\end{lemma}
\begin{proof}[Proof of ~\Cref{thm:clt-crt-logit}]
  The following proof is an adaptation from \cite{ning_general_2017}.
  Notice that our version of the proof is shorter, with specific consideration
  on sparse logistic regression, and with elaboration on the convergence rate
  of the decorreleated test score, which is missing from
  \cite{ning_general_2017}.

  Denote $\hat{\bv} \egaldef (1, (\betadx)^{\top})$, then the decorrelated test score can
  be written in more general from as
  \begin{equation}
    \label{eq:test-score-general}
    T_j^{\text{decorr}} = n^{1/2} \hat{\bI}_{j \mid -j}^{-1/2 } \left(\nabla_j\ell(\hat{\bm\beta}) -
      (\betadx)^{\top}\nabla_{\bm\beta_{-j}}\ell(\hat{\bm\beta}) \right) =
    n^{1/2} \hat{\bI}_{j \mid -j}^{-1/2 }
    \hat{\bv}^{\top}\nabla\ell(\hat{\bm\beta}) \ .
  \end{equation}
  Moreover, denote $\betanull \egaldef (0, \hat{\bm\beta}_{-j})$ and
  $\bv^* \egaldef (1, \wjtrue)$, then we have, under the null hypothesis,
  \begin{align*}
    n^{1/2}\abs{\hat{\bv}^{\top}\nabla(\betanull) -
             \bv^{*\top}\nabla\ell(\betatrue)}
          \leq \underbrace{n^{1/2}\abs{\bv^{*\top}\{\nabla\ell(\betatrue) -
             \nabla\ell(\betanull)\}}}_{A_1}
             + \underbrace{n^{1/2}\abs{(\hat{\bv} - \bv^*)^{\top} \nabla\ell(\betanull)}}_{A_2}
  \end{align*}
  where we use triangle inequality with the last step.
  By Taylor expansion, and from Lemma~\ref{lemma:smoothness}, we have
  \begin{align*}
    A_1 &\leq n^{1/2}\abs{\bv^{*\top}\nabla^2\ell(\betatrue)(\betanull -
                \betatrue)}\\
        &\leq n^{1/2}\norm{\betanull - \betatrue}_1
          \norm{\bv^{*\top}\nabla^2\ell(\betatrue)}_{\infty}\\
        &\precsim \dfrac{s^*\log(p)}{\sqrt{n}}
  \end{align*}
  where the second inequality is by Hold\"er inequality, 
  and the last inequality us due to Lemma~\ref{lemma:error-bound-beta}
  and~\ref{lemma:gradient-hessian}.
  Similarly, we can bound $A_2$, by using Lemma~\ref{lemma:error-bound-w} and
  Lemma~\ref{lemma:smoothness}
  \begin{align*}
    A_2 &\leq n^{1/2}\abs{(\hat{\bv} - \bv^*)^{\top}
                \nabla\ell(\betatrue)} \\
              &\leq n^{1/2}\norm{\hat{\bv} - \bv^*}_1
                \norm{\nabla\ell(\betatrue)}_{\infty} \precsim \dfrac{(s^* \vee s') \log p}{n}
  \end{align*}
  This implies that,
  \begin{equation}
    \label{eq:non-asymp-bound-fisher-score}
    n^{1/2}\abs{\hat{\bv}^{\top}\nabla(\betanull) -
      \bv^{*\top}\nabla\ell(\betatrue)} \precsim n^{-1/2}(s^* \vee s')\log(p) \ .
  \end{equation}
  The remaining part of the proof is to bound
  $\hat{\bI}_{j \mid -j} - \bI_{j \mid -j}$, where, by definition
  \begin{equation*}
    \bI_{j \mid -j} = \bbE \left\{
      g''(\bX_{i, *}\bm\beta^0)\left[\bX_{i,j} - \bX_{i,-j}\wjtrue\right] \ \bX_{i,j}  \right\}
  \end{equation*}
  Evaluating the difference between $\hat{\bI}_{j \mid -j}$ and
  $\bI_{j \mid -j}$ gives
  \begin{multline*}
    \hat{\bI}_{j \mid -j}$ - $\bI_{j \mid -j} \\= 
\dfrac{1}{n} \sum_{i=1}^n g''(\bX_{i, *}\hat{\bm\beta}) \left[\bX_{i,j}
  - \bX_{i, -j} \betadx\right] \ \bX_{i,j} - \bbE
\left\{g''(\bX_{i, *}\bm\beta^0) \left[\bX_{i,j} - \bX_{i,-j}\wjtrue \right] \ \bX_{i,j}  \right\} \\
= \left( \dfrac{1}{n} \sum_{i=1}^n g''(\bX_{i, *}\hat{\bm\beta})
  \bX_{i,j}^2 - \bbE \left\{g''(\bX_{i, *}\bm\beta^0) \bX_{i,j}^2 \right\} \right) \\ +
\left( \dfrac{1}{n} \sum_{i=1}^n g''(\bX_{i, *}\hat{\bm\beta})
  \bX_{i,-j}\betadx \ \bX_{i,j} - 
    \bbE \left\{ g''(\bX_{i, *}\bm\beta^0) \bX_{i,-j}\wjtrue
      \ \bX_{i,j}  \right\}  \right) \\
  \leq \underbrace{\left( \dfrac{1}{n} \sum_{i=1}^n g''(\bX_{i, *}\hat{\bm\beta})
  \bX_{i,j}^2 - \bbE \left\{g''(\bX_{i, *}\bm\beta^0) \bX_{i,j}^2
  \right\} \right)}_{C} + 
\underbrace{\left\lvert \dfrac{1}{n} \sum_{i=1}^n g''(\bX_{i,
      *}\hat{\bm\beta})\bX_{i, -j}(\betadx - \wjtrue) \ \bX_{i,j} \right\rvert}_{B_1} \\
+ \underbrace{\left\lvert \dfrac{1}{n} \sum_{i=1}^n
  [g''(\bX_{i,*}\betatrue) - g''(\bX_{i,*}\hat{\bm\beta})]\bX_{i, -j}\wjtrue \ \bX_{i,j}
\right\rvert}_{B_2} + \\
\underbrace{\left\lvert \dfrac{1}{n} \sum_{i=1}^n g''(\bX_{i, *}\bm\beta^0)
    \bX_{i,-j}\wjtrue \ \bX_{i,j} - \bbE \left\{ g''(\bX_{i,*}\bm\beta^0)
       \bX_{i,-j}\wjtrue \ \bX_{i,j}  \right\} \right\rvert}_{B_3},
\end{multline*}
where the last step follows triangle inequality.

We have, by Cauchy-Schwartz inequality, by Lemma~\ref{lemma:error-bound-w}, and
by the fact that $g''(x) \in (0,1)$ for every $x \in \bbR$; and $\bX_{i, -j}$,
$\bX_{i, j}$ is sub-exponential by \Cref{astn:regularity-glm}:
\begin{align*}
  B_1
  &\leq \sqrt{\left( \dfrac{1}{n} \sum_{i=1}^n g''(\bX_{i,*}^{\top}\hat{\bm\beta})((\betadx - \wjtrue)^{\top}\bX_{i, -j})^2
        \right) \left( \dfrac{1}{n} \sum_{i=1}^n g''(\bX_{i, *}^{\top}\hat{\bm\beta})\bX_{i, j}^2 \right)} \\ 
  &\precsim \sqrt{\dfrac{(s^* \vee s')\log(p)}{n}}.
\end{align*}

Similarly, to bound $B_2$, we have, again by Cauchy-Schwartz inequality,
\begin{align*}
  B_2 &\leq \sqrt{\dfrac{1}{n} \sum_{i=1}^n
  [g''(\bX_{i,*}\bm\beta^0) - g''(\bX_{i,*}\hat{\bm\beta})]^2
  \left(\bX_{i, -j} \wjtrue \ \bX_{i,j} \right)^2} \\
  &\leq \sqrt{\dfrac{1}{n} \sum_{i=1}^n 
    [g''(\bX_{i,*}\betatrue)\bX_{i,*}(\hat{\bm\beta} - \betatrue)]^2
    \left(\bX_{i, -j}\wjtrue \ \bX_{i,j} \right)^2} \ ,
  \\
\end{align*}
where the second inequality comes from using the self-concordance property of
the sigmoid function (discussed at length in \cite{bach_self-concordant_2010}
and extended further in \cite{ostrovskii_finite_2021}), that is,
$\abs{g''(t_1) - g''(t)} \leq \abs{t_1 - t}g''(t)$ for a fixed constant $t$,
and for every $t_1 \in \bbR$ such that $t_1$ converges to $t$, with
$t_1 = \hat{\bm\beta}$, and $t = \betatrue$.
By \Cref{astn:regularity-glm}-A3 that $\bX_{i,j}$ is sub-exponential, applying
Bernstein inequality leads to
\begin{equation*}
  B_2 \precsim \sqrt{\dfrac{s^* \log p}{n}}.
\end{equation*}

To bound $B_3$, by direct application of Hoeffding inequality, we have
$B_3 \precsim \sqrt{\dfrac{(s^* \vee s')\log p}{n}}$.
This implies
\begin{equation}
  \label{eq:non-asymp-bound-fisher-info}
  \abs{\hat{\bI}_{j \mid -j} - \bI_{j \mid -j}} \precsim \sqrt{\dfrac{(s^*
      \vee s')\log p}{n}} \ .
\end{equation}
Putting Equation~\eqref{eq:non-asymp-bound-fisher-score}
and~\eqref{eq:non-asymp-bound-fisher-info} together, we have, under null
hypothesis,
\begin{equation*}
T_j^{\text{decorr}} \xrightarrow{\cD} n^{1/2} \ \bI_{j \mid -j}^{-1/2}
\bv^{*\top}\nabla\ell(\betatrue) \egaldef T_j^*\ ,
\end{equation*}
with convergence rate $\cO(n^{-1/2})$.
Finally, by noting that we can decompose
$\nabla\ell(\betatrue) = \frac{1}{n}\sum_{i=1}^n\nabla \ell_i(\bm\beta^0)$, and
each $\ell_i(\bm\beta^0)$ has bounded first, second, and third moment, a direct
application of Berry-Esseen theorem give convergence in distribution of $T^*_j$
to a standard normal law, with rate $\cO(n^{-1/2})$.

We also arrive at the second conclusion of \Cref{thm:clt-crt-logit} by noting
that it is a straightforward by-product of the result on normality of the
distribution of decorrelated test score under null hypothesis, based on the
formula for the p-values of CRT-logit algorithm.
\end{proof}
\begin{proof}[Proof of \Cref{thm:asymp-fdr-control-crt-logit}]
  The proof of this theorem is a straightforward adaptation from
  \cite{benjamini_control_2001}.
  For shorter notation, we denote $\hatS \egaldef \Sbycrt$ and
  $\hatK \egaldef \widehat{k}_{BY} $.
  If we denote
  $\bar{\alpha} \egaldef \dfrac{\alpha}{p\sum^p_{i=1}1/i} \in (0, 1)$, then
  step 1 in the procedure defined in \Cref{df:by} is equivalent with finding
  $\hatK$ such that
  \begin{equation}
    \label{eq:k-hat}
    \hatK = \max \left\{ k \in [p] \mid \hat{p}_{(k)}
        \leq k\bar{\alpha} \right\}.
  \end{equation}
  For every $i,j,k \in [p]$, let us define
  \begin{equation}
    p_{i,j,k} = 
    \begin{cases}
      \bbP \left(\hatP_i \in \left( (j-1) \bar{\alpha}, j\bar{\alpha} \right] \ ,
        i \in \widehat{S} \text{ and } \lvert \widehat{S} \rvert = k \right)
      \qquad &\text{if } j \geq 2
      \\
      \bbP \bigl( \hatP_i \in [0 , \bar{\alpha} ] \, , \, i \in \widehat{S} \text{
        and } \lvert \widehat{S} \rvert = k \bigr) \qquad &\text{if } j = 1.
    \end{cases}
    \label{eq:p-ikj}
  \end{equation}
  Then, since $i \in \hatS$ and $\abs{\hatS} = k$ implies that
  $\hatP_i \leq \hatP_{\hatK} \leq \hatK\bar{\alpha} = k\bar{\alpha}$, we have
  \begin{align*}
    \frac{\lvert \widehat{S} \cap \cS^c \rvert}{\lvert \widehat{S} \rvert \vee 1}
    &= \sum_{k=1}^p \ind{\abs{\widehat{S}} = k} \dfrac{\sum_{i \in \cS^c} \ind{i \in \widehat{S}}}{k} 
      = \sum_{i \in \cS^c} \sum_{k=1}^p \dfrac{1}{k} \ind{\abs{\widehat{S}} = k \text{ and } i \in \widehat{S}}
    \\
    &= \sum_{i \in \cS^c} \sum_{k=1}^p \frac{1}{k} \ind{\abs{\widehat{S}} = k
      \text{ and } i \in \widehat{S} \text{ and } 0 \leq \hat{p}_i \leq k
      \bar{\alpha} } \ .
  \end{align*}
  Taking an expectation and writing that
  \begin{equation*}
    \ind{0 \leq \hatP_i \leq k \bar{\alpha} } 
    = \ind{\hatP_i \in [0,\bar{\alpha} ]} +
    \sum_{j = 2}^k \ind{\hatP_i \in ((j-1)\bar{\alpha} ,j\bar{\alpha} ]} \ ,
  \end{equation*}
  we get
  \begin{align*}
    \bbE \left[ \frac{\abs{\hatS \cap \cS^c}}{\abs{\widehat{S}} \vee 1} \right]  
    &= 
      \sum_{i \in \cS^c} \sum_{k=1}^p \frac{1}{k} \sum_{j=1}^k p_{i,j,k} 
      = \sum_{i \in \cS^c} \sum_{j=1}^p \sum_{k=j}^p \frac{1}{k} p_{i,j,k}  \\
    &\leq \sum_{i \in \cS^c} \sum_{j=1}^p \sum_{k=j}^p \frac{1}{j} p_{i,j,k} 
      = \underbrace{\sum_{j=1}^p \frac{1}{j} \sum_{i \in \cS^c} \sum_{k=j}^p p_{i,j,k}}_{A} \ .
  \end{align*}
  Denote
  $F(j) \egaldef \sum_{i \in \cS^c} \sum_{j'=1}^j \sum_{k=1}^p
  p_{i,j',k}$ for all $j \in \{1, \ldots, p\}$, 
  and remark that $p_{i,j',k} = 0$ if $j'>k$, by definition of~$\Sbycrt$. 
  We then have
  \begin{equation*}
    A = F(1) + \sum_{j=2}^p \frac{1}{j} \bigl[ \, F(j) -
    F(j-1) \bigr] = \sum_{j=1}^{p-1} \left( \frac{1}{j} - \frac{1}{j+1}
    \right) F(j) + \frac{F(p)}{p} \ .
  \end{equation*}    
  This leads to
  \begin{equation}
    \bbE \left[ \frac{\abs{\hatS \cap \cS^c} }{\abs{\hatS} \vee 1} \right]  
    \leq \sum_{j=1}^{p-1} \left( \frac{1}{j} - \frac{1}{j+1} \right) F(j) + \frac{F(p)}{p}
  \end{equation}
  By the definition of $p_{i,j,k}$ in Eq.~\eqref{eq:p-ikj}, we have
  \begin{equation*}
    F(j) = \sum_{i \in \cS^c} \bbP( \hatP_i \leq j \bar{\alpha}  \text{ and } i \in \hatS )
    \leq \sum_{i \in \cS^c} \bbP( \hatP_i \leq j \bar{\alpha} ). 
  \end{equation*}
  Therefore
  \begin{equation*}
    \bbE \left[ \frac{\abs{\hatS \cap \cS^c} }{\abs{\hatS} \vee 1} \right]  
    \leq \sum_{i \in \cS^c} \sum_{j=1}^{p-1} \frac{\bbP(\hatP_i \leq j
        \bar{\alpha})}{j(j + 1)} + \sum_{i \in \cS^c} \frac{\bbP( \hatP_i \leq p\bar{\alpha} )}{p}
  \end{equation*}

  Taking the limit where $n \to \infty$ and $p$ fixed, we have, using the
  result in \Cref{thm:clt-crt-logit},
  \begin{align*}
    \limsup_{n \to \infty} \bbE \left[ \frac{\abs{\hatS \cap \cS^c}}{\abs{\hatS} \vee 1} \right]
    &\leq \sum_{i \in \cS^c} \left(\sum_{j=1}^{p-1} \frac{1}{j+1} + 1 \right)\bar{\alpha}
    \\
    &= \left( \sum_{j=1}^{p} \frac{1}{j} \right) \abs{\cS^c} \bar{\alpha} \ .
\end{align*}
We conclude the proof by noting that
$\bar{\alpha} \egaldef \dfrac{\alpha}{p\sum^p_{j=1} 1/j}$.
\end{proof}

\section{Controlling False Discovery Rate Procedures}

\begin{df}[Benjamini-Hochberg procedure \citep{benjamini_controlling_1995}]
  \label{df:by}
  Let $\alpha \in (0, 1)$ be the predefined FDR control level.
  Let $\hat{p}_1, \dots, \hat{p}_m$ be output p-values from inference
  algorithm, \eg Algorithm~\ref{alg:CRT-logit}.
  We reorder them ascendingly, denoted by
  $\hat{p}_{(1)} \leq \hat{p}_{(2)} \leq \dots \leq \hat{p}_{(p)}$ and
  $\cH_0^{(1)}, \dots, \cH_0^{(p)}$, then
  \begin{enumerate}
  \item Find $\widehat{k}_{BH}$ such that
    \begin{equation*}
      \widehat{k}_{BY} = \max \left\{ k \in [p] \mid \hat{p}_{(k)} \leq
        \dfrac{k\alpha}{p} \right\}.
    \end{equation*}
  \item If $\widehat{k}_{BH}$ exists, take
    $\widehat{S} = \{j \in [p]: \hat{p}_{(j)} \leq \hat{p}_{\widehat{k}_{BH}}\}$.
    Otherwise $\widehat{S} = \emptyset$.
  \end{enumerate}
\end{df}

\begin{df}[Benjamini-Yekutieli procedure \citep{benjamini_control_2001}]
  \label{df:by}
  Let $\alpha \in (0, 1)$ be the predefined FDR control level.
  Let $\hat{p}_1, \dots, \hat{p}_m$ be output p-values from
  Algorithm~\ref{alg:CRT-logit}.
  We reorder them ascendingly, denoted by
  $\hat{p}_{(1)} \leq \hat{p}_{(2)} \leq \dots \leq \hat{p}_{(p)}$ and
  $\cH_0^{(1)}, \dots, \cH_0^{(p)}$, then
  \begin{enumerate}
  \item Find $\widehat{k}_{BY}$ such that
    \begin{equation*}
      \widehat{k}_{BY} = \max \left\{ k \in [p] \mid \hat{p}_{(k)} \leq
        \dfrac{k\alpha}{p\sum^p_{i=1}1/i} \right\}.
    \end{equation*}
  \item If $\widehat{k}_{BY}$ exists, take
    $\widehat{S} = \{j \in [p]: \hat{p}_{(j)} \leq \hat{p}_{\widehat{k}_{BY}}\}$.
    Otherwise $\widehat{S} = \emptyset$.
  \end{enumerate}
\end{df}

\section{Setting the $\ell_1-$Regularization Parameter of the $\bX_{*, j}$-distillation}
\label{ssec:setting-l1}
A core issue is the dependency of the statistical power and FDR of CRT-logit on
the $\ell_1-$ regularization parameter $\lambda_{dx}$ when doing Lasso
distillation on $x_j$ in Eq.~\eqref{eq:scaled-betadx}.
One might choose the heuristic value
$\lambda_{\text{univ}} = \sqrt{n^{-1}\log p}$ with theoretical validity, as
suggested in \cite{ning_general_2017,van_de_geer_asymptotically_2014}.
However, experimental results in Fig.~\ref{fig:fdr-power-n-lambda} show that at
$\lambda_{dx} = \lambda_{\text{univ}}$ (or
$\log_{10}\lambda/\lambda_{\text{univ}} = 0.0$ with the labeling of the
figure), we do not have the best possible FDR/Power with CRT-logit inference.
For this experiment, we average the inference results of 100 simulations (with
similar setting in Section~\ref{ssec:qqplot}) for different values of $n$ and
$\lambda_{dx}$, with $p$ fixed.
There is a clear phase transition in both FDR and average power when the
regularization parameter $\lambda_{dx}$ increases.
In other words, we have found empirically that both FDR and power of the method
are sensitive to the $\ell_1-$regularization parameter.
Preferably, one wants to return a high statistical power while controlling FDR under predefined level.
Hence, it is necessary to choose $\lambda_{dx}$ wisely.
In a more practical scenario, we advise to use cross-validation for
$\bX_{*, j}$-distillation operator, as defined by Eq.~\eqref{eq:scaled-betadx}.
This means we would have to find $p$ different values of $\lambda_{dx}$ with
cross-validation, and we reemphasize the importance of the screening step to
reduce the number of computations.
\begin{figure}[h]
  \centering
  \includegraphics[width=\textwidth]{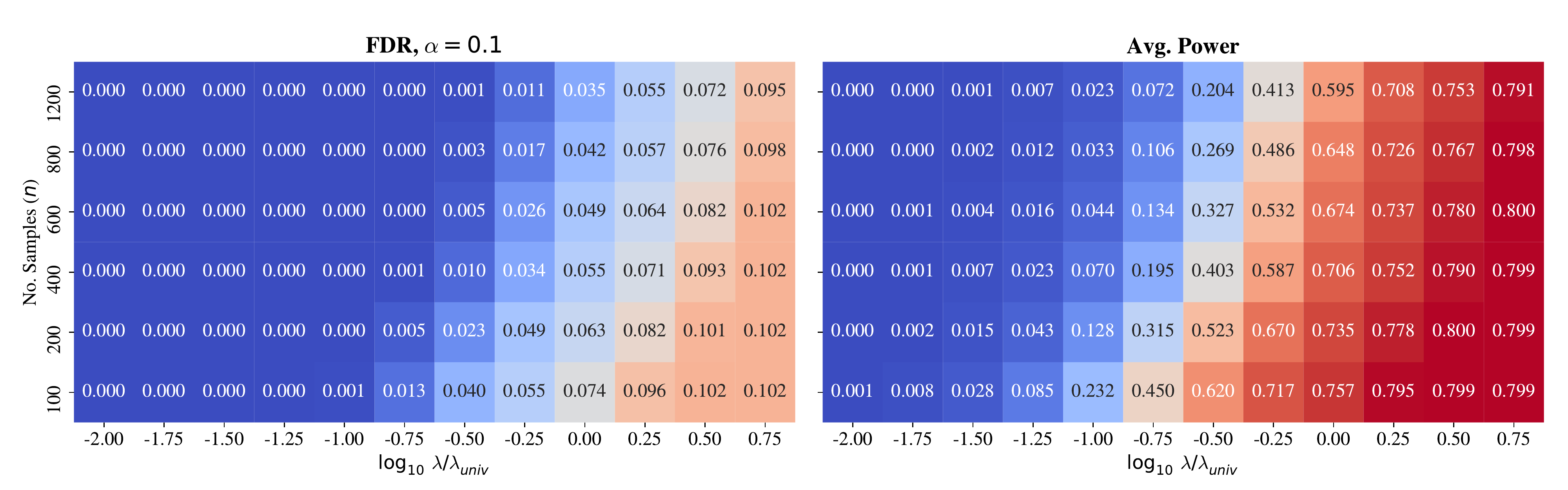}
  \caption{\textbf{FDR/Average Power of 100 runs of simulations while varying
      the number of samples and $\ell_1$ regularization parameter and fixing
      the number of variables}. Note: $\lambda_{dx}$ is scaled with the factor
    $\lambda_{\text{univ}} = \sqrt{\log(p) / n}$, \eg the first value for
    regularization grid is $\lambda_{dx} = 10^{-2} \lambda_{\text{univ}}$.
    Default parameter (similar settings in Section~\ref{ssec:qqplot}):
    $p=400$, SNR=$3.0$ (signal-to-noise ration), $\rho=0.5$ (feature correlation),
    $\kappa=0.05$ (sparsity). FDR is controlled at level $\alpha=0.1$.}
  \label{fig:fdr-power-n-lambda}  
\end{figure}

\section{Pseudocode for CRT-logit and Related Algorithms}
\label{sec:pseudocode}
\begin{algorithm}[h]
  \small
  \SetAlgoLined
{\textbf{INPUT} dataset $(\bX, \by)$, with $\bX \in \bbR^{n \times p}$, $\by \in
  \bbR^n$, number of sampling runs $B$, test statistic $T_j$, conditional
distribution $P_{j \mid -j}$ for each $j = 1, \dots, p$} ; \\
{\textbf{OUTPUT} vector of p-values $\{\hat{p}_j\}_{j=1}^p$};  \\
\For{$j = 1, 2, \dots, p$}{
  Compute test statistics $T_j$ for original variable; \\
  \For{$b = 1, 2, \dots, B$}{
    1. Generate $\tilde{\bX}_{*, j}^{(b)}$, a knockoff sample from $P_{j
    \mid -j}$; \\
    2.  Compute $\tilde{T}_j^{(b)}$ for knockoff variables;
  }
  Compute the empirical p-value
  $$
    \hatP_j = \dfrac{1 + \sum_{b=1}^{B}
    \textbf{1}_{\tilde{T}_j^{(b)} \geq T_j}}{1 + B}
  $$
}
\caption{Conditional Randomization Test \citep{candes_panning_2018}}
\label{alg:CRT}
\end{algorithm}
\begin{algorithm}[h]
  \small
  \SetAlgoLined
{\textbf{INPUT} dataset $(\bX, \by)$, $X \in \bbR^{n \times p}$, $\by \in \bbR^{n}$, test
statistic $T_j$ for each $j = 1, \dots, p$}; \\
{\textbf{OUTPUT} vector of p-values $\{p_j\}_{j=1}^p$}; \\
$\hat{\cS}^{\text{SCREENING}} = \{j: j \in [p], \hat{\beta}^{\text{MLE}}_j \neq
0\}$ \tcp{Using Eq.~\eqref{eq:mle}}
\For{$j \in \hat{\cS}^{\text{SCREENING}}$}{
  1. Distill information of  $\bX_{-j}$ to $\bX_{*, j}$ and to $\by$ by finding:
  \begin{itemize}
  \item
    $\betady(\lambda) \gets \texttt{solve\_sparse\_logistic\_cv}(\bX_{-j}, \by
)$ \tcp{Using
  Eq.~\eqref{eq:mle}
}
  \item $\betadx(\lambda) =
    \text{argmin}_{\boldsymbol{\beta} \in \mathbb{R}^{p - 1}} \dfrac{1}{2}
    \left\lVert {\bX_{*, j} - \mathbf{X}_{-j}\boldsymbol{\beta}}
    \right\rVert_2^2 + \lambda_{dx}\left\lVert {\boldsymbol{\beta}}
    \right\rVert_1$ \tcp{with $\lambda_{dx}$ set using cross-validation}
  \end{itemize}
  2. Obtain test statistic:
  $$
  T_j = 
      \sqrt{n} \dfrac{\inner{\by - \bX_{-j}\betady \ , \bX_{*, j} - \bX_{-j}\betadx}}{
      \left\lVert{\by - \bX_{-j} \betady} \right\rVert_2
      \left\lVert {\bX_{*, j} - \bX_{-j}\betadx} \right\rVert_2
    }
  $$
  3. Compute (two-sided) p-value
  $$
  \hatP_j = 2\left[1 - \Phi\left( T_j \right) \right]
  $$
  }
  \caption{Lasso-Distillation Conditional Randomization Test
    \citep{liu_fast_2020}}
  \label{alg:dCRT}
\end{algorithm}
\begin{algorithm}[h]
  \small
  \SetAlgoLined
{\textbf{INPUT} dataset $(\bX, \by)$, with $\bX \in \bbR^{n \times p}$, $\by \in
  \bbR^n$, number of sampling runs $B$, test statistic $T_j$, conditional
distribution $P_{j \mid -j}$ for each $j = 1, \dots, p$}, empirical risk $L(\cdot)$ ; \\
{\textbf{OUTPUT} vector of p-values $\{\hat{p}_j\}_{j=1}^p$};  \\
$(\bX_{\text{train}}, \by_{\text{train}}), (\bX_{\text{test}},
\by_{\text{test}}) \gets \texttt{data\_splitting}(\bX, \by)$; \\
$\hat{f}_{\theta} \gets \texttt{model\_fitting}(\bX_{\text{train}}, \by_{\text{train}})$; \\

\For{$j = 1, 2, \dots, p$}{
  $T_j \gets L(\bX_{\text{test}}, \by_{\text{test}}, \hat{f}_{\theta}(\bX_{\text{test}}))$; \\  
  \For{$b = 1, 2, \dots, B$}{  
    1. Generate $\tilde{\bX}_{*, j}^{(b)} \sim P_{j \mid -j}$; \\
    2. $\tilde{T}_j^{(b)} \gets L(\tilde{\bX}_{*, j}^{(b)}, \by_{\text{test}}, \hat{f}_{\theta}(\tilde{\bX}_{*, j}^{(b)}))$;
  }
  Compute the empirical p-value
  $$
    \hatP_j = \dfrac{1 + \sum_{b=1}^{B} \textbf{1}_{\tilde{T}_j^{(b)} \geq T_j}}{1 + B}
  $$
}
\caption{Holdout Randomization Test \citep{tansey_holdout_2018}}
\label{alg:HRT}
\end{algorithm}

\section{Time complexity of Related Methods}

We present the time complexity of benchmarked methods in
Table~\ref{table:time-complexity}.
\begin{table}[h]
  \small
  \centering
  \caption{\textbf{Time complexities of related methods with CRT-logit}, where
    $p$ is the dimension size (number of variables), $B$ is the number of
    sampling runs, and $\hat{k} \ll p$ the cardinality of the screening set
    (see Section~\ref{sec:pseudocode} for more details).}
\begin{tabular}{lp{3cm}p{6.0cm}}
  \toprule
  \textbf{Methods} & \textbf{Time (Iteration) Complexity} & \textbf{References} \\
  \midrule
  Debiased Lasso & $\cO(p^4)$ & \cite{zhang_confidence_2014,van_de_geer_asymptotically_2014,javanmard_confidence_2014}\\
  Knockoff Filter & $\cO(p^3)$ & \cite{barber_controlling_2015,candes_panning_2018}\\
  CRT & $\cO(Bp^4)$ & \cite{candes_panning_2018}\\
  HRT & $\cO(p^3 + Bp^2)$ & \cite{tansey_holdout_2018}\\  
  dCRT (with screening )& $\cO(\hat{k}p^3)$ & \cite{liu_fast_2020} \\
  \textbf{CRT-logit (with screening)} &  $\cO(\hat{k}p^3)$ & \textbf{(this work)} \\
  \bottomrule
\end{tabular}
\label{table:time-complexity}
\end{table}

\section{Additional Details on Experiments in Section~\ref{sec:empirical}}
\label{sec:details-experiment}

\subsection{Preprocessing of the brain-imaging dataset}
The Human Connectome Project dataset (HCP) is a collection of brain imaging data
on healthy young adult subjects with age ranging from 22 to 35.
The participants performed different tasks while being scanned by a magnetic
resonance imaging (MRI) device to record blood oxygenation level dependent
(BOLD) signals of the brain.
The aim of this analysis is to investigate which areas of the brain
can predict cognitive activity across participants, while taking into
account the information from other brain regions.
The brain imaging modalities include, among others, resting-state fMRI (R-fMRI)
and task-evoked fMRI (T-fMRI).
In this work, we only deal with decoding the task-evoked fMRI dataset.
The four classification problems we are working with are as follows.
\begin{itemize}
\item Relational: predict whether the participant matches figures or identified
  feature similarities.
\item Gambling: predict whether the participant gains or loses gambles.
\item Emotion: predict whether the participant watches an angry face or a geometric shape.  
\item Social: predict whether the participant watches a movie with social behavior or not.
\end{itemize}
To perform dimension reduction, our goal is to apply a clustering scheme that
keeps the spatial structure of the data.
This is achieved with data-driven parcellation along with a spatially
constrained clustering algorithm, following the conclusions by
\cite{varoquaux_small-sample_2012} and \cite{thirion_which_2014}.
The hierarchical clustering scheme that we use recursively merges pair of
clusters of features based on a criterion that minimized the within-cluster
variance.
This algorithm is implemented in \texttt{scikit-learn}
\cite{pedregosa_scikit-learn_2011}, a popular package for applied machine
learning.

\section{Example of decoding maps in semi-realistic brain-analysis experiment
  of Section~\ref{ssec:simu-hcp}}

\begin{figure}[h]
  \centering
  \begin{subfigure}[c]{0.35\linewidth}
    \centering
    \includegraphics[width=\textwidth]{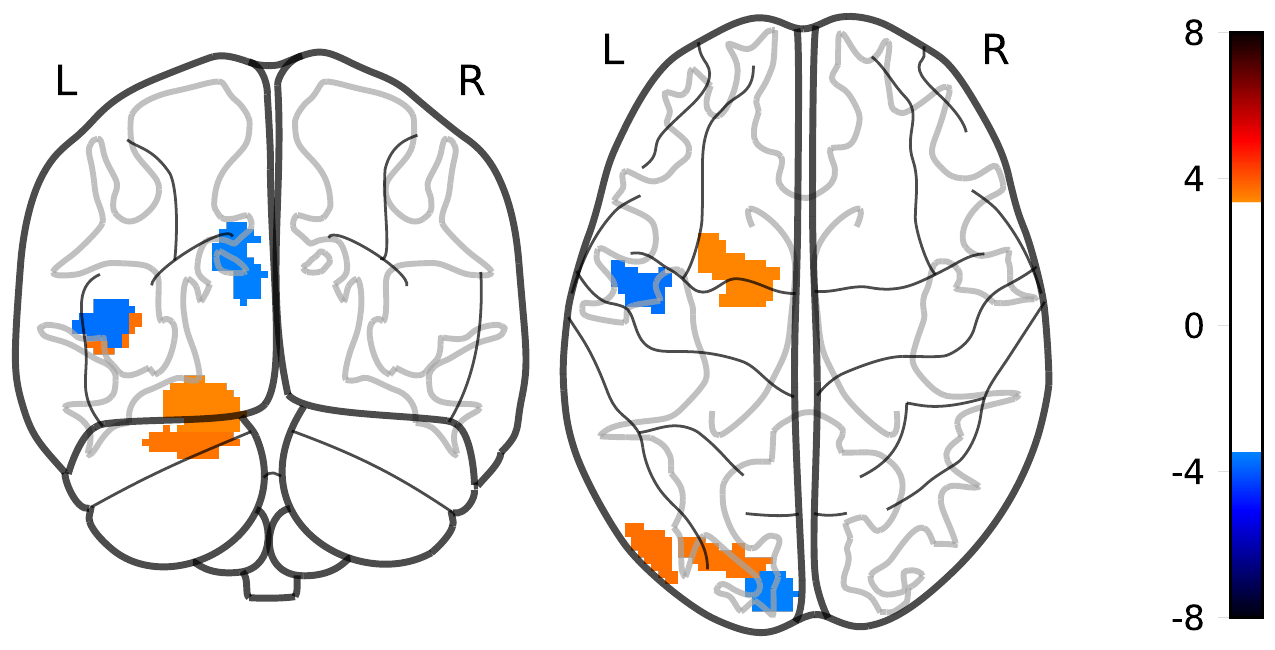}
    \caption{Debiased Lasso (\texttt{dlasso})}
  \end{subfigure}
  \begin{subfigure}[c]{0.28\linewidth}
    \centering
    \includegraphics[width=\textwidth]{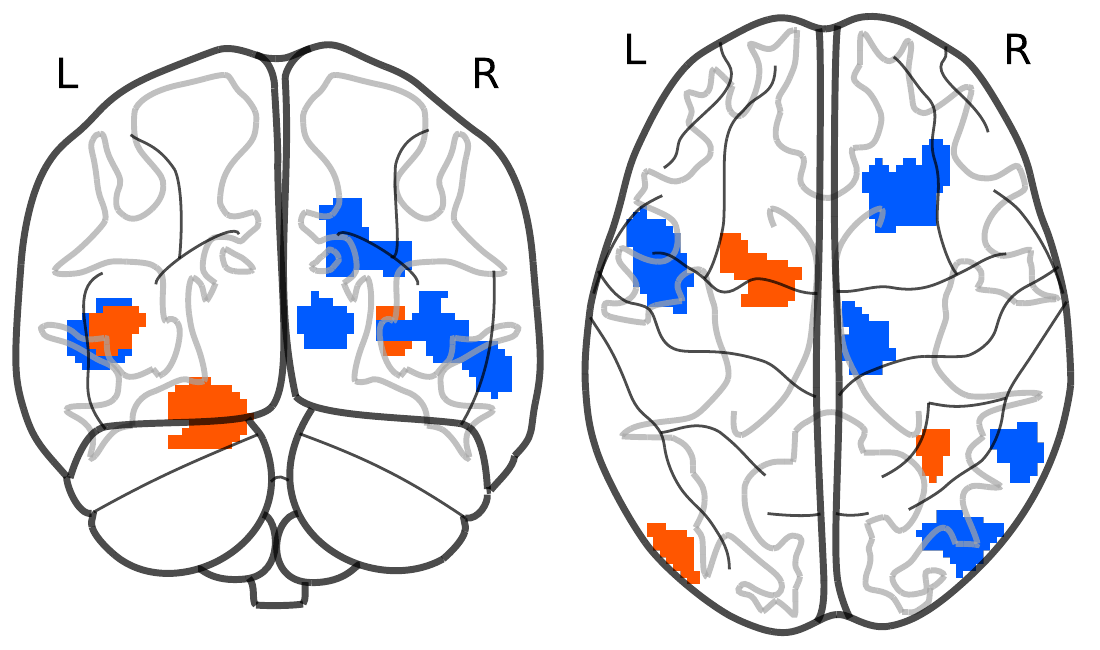}
    \caption{Knockoffs (\texttt{KO})}
  \end{subfigure}
  \begin{subfigure}[c]{0.28\linewidth}
    \centering
    \includegraphics[width=\textwidth]{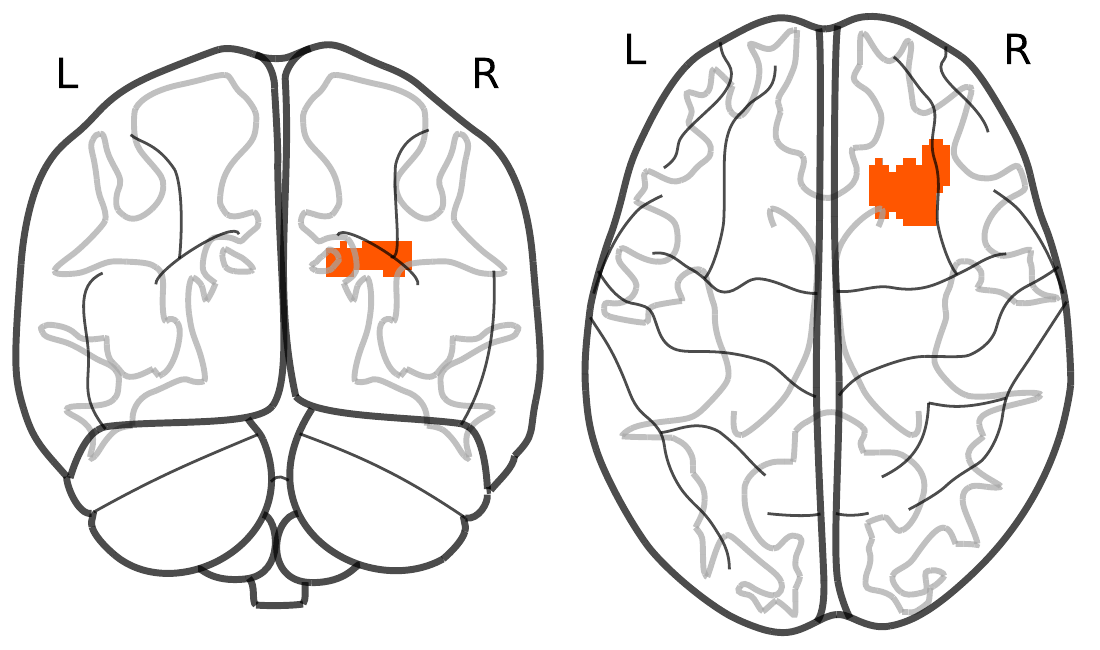}
    \caption{CRT (\texttt{CRT}), $B=500$}
  \end{subfigure}
  \begin{subfigure}[c]{0.28\linewidth}
    \centering
    \includegraphics[width=\textwidth]{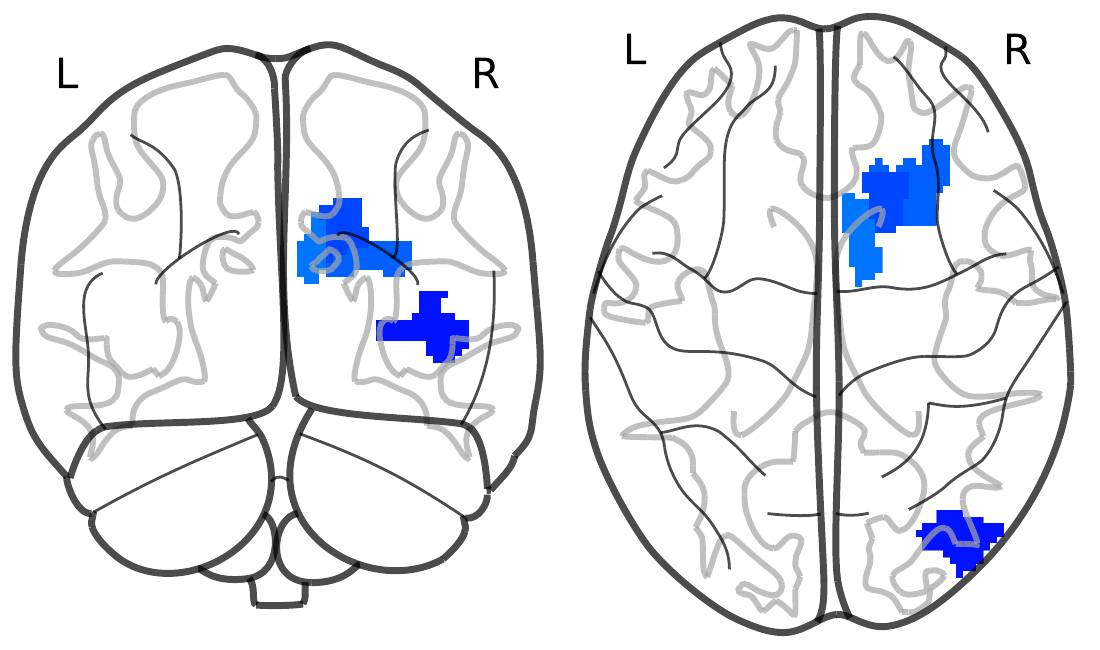}
    \caption{Distilled-CRT (\texttt{dCRT})}
  \end{subfigure}
  \begin{subfigure}[c]{0.35\linewidth}
    \centering
    \includegraphics[width=\textwidth]{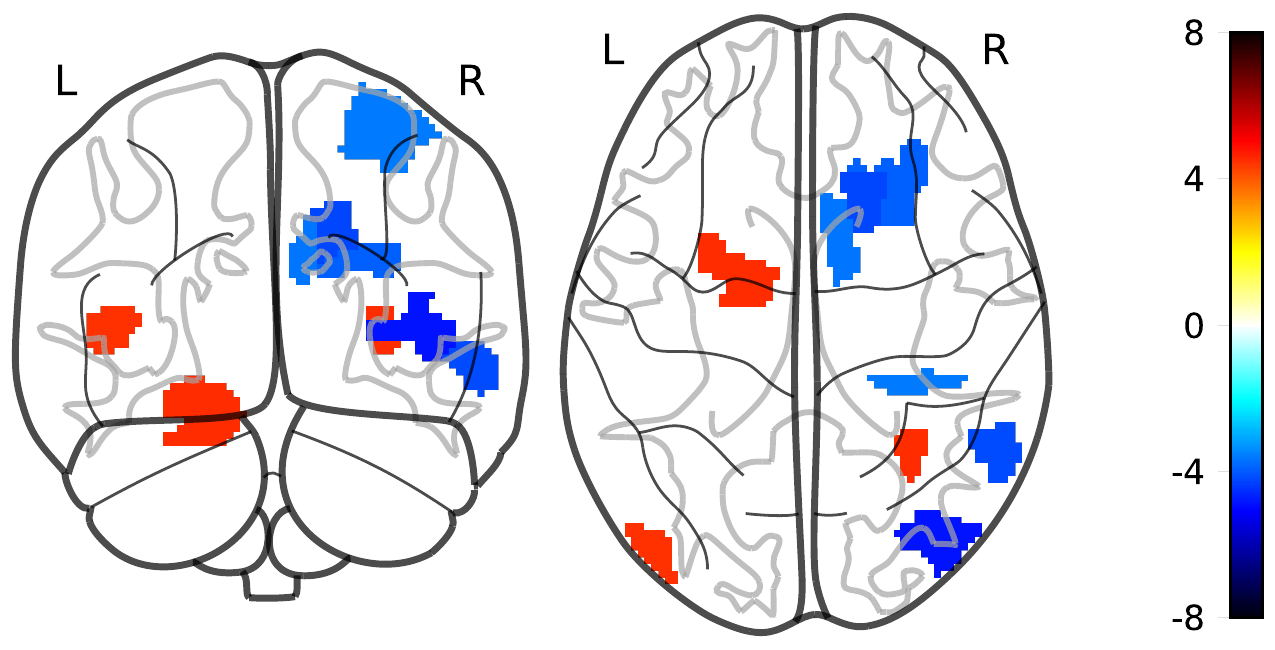}
    \caption{CRT-logit (our method)}
  \end{subfigure}  
  \caption{\textbf{Decoding maps of Relational task in semi-realistic HCP900
      experiment, using 400 subjects and dimension reduction to 1000 clusters
      (\ie one random seed for generating labels $\by$)}. We omit Holdout
    Randomization Test (\texttt{HRT}) as the method does not select any brain
    region.  For \texttt{dlasso}, \texttt{dCRT} and \texttt{CRT-logit}, we plot
    the test-statistics; for \texttt{KO} the sign of selected coefficients, and
    for \texttt{CRT} the $-log_{10}$ of the empirical p-values.  }
  \label{fig:qqplot}
\end{figure}

\end{document}